\RequirePackage{fix-cm}
\documentclass[smallcondensed,envcountsame]{svjour3}     
\smartqed  
\usepackage{graphicx}
\usepackage{times}
\usepackage{helvet}
\usepackage{courier} 
\usepackage{url}
\usepackage{amsmath}
\usepackage{graphics}
\usepackage{epsfig}
\usepackage{float}
\usepackage{multirow}
\usepackage{booktabs}       
\usepackage{amsfonts}       
\usepackage{nicefrac}       
\usepackage{microtype}      
\usepackage{amssymb}
\usepackage{amsmath}
\usepackage{stmaryrd}
\usepackage{multicol}
\usepackage{color}

\usepackage{algorithm}
\usepackage{algorithmic}

\usepackage{caption}
\usepackage{float}

\newcommand{\bfunc}[1]{\llbracket#1\rrbracket}
\newcommand{\Mv}[1]{\mathbf{#1}}
\newcommand{\Mt}[1]{\mathtt{#1}}
\newcommand{\Mvh}[1]{\hat{\mathbf{#1}}}
\newcommand{\Mvb}[1]{\bar{\mathbf{#1}}}

\newcommand{\Mvt}[1]{\tilde{\mathbf{#1}}}

\newcommand{\argmin}[1]{\underset{#1}{\arg\min}\quad}

\newcommand{\journalminor}[1]{{\color{black}#1}}

\newcommand{\journallast}[1]{{\color{black}#1}}

\newcommand{\setre}[1]{\renewcommand{\algorithmicrequire}{\textbf{#1}}}
\renewcommand{\algorithmicrequire}{\textbf{Input:}}

\begin{document}

\title{Dynamic Principal Projection for  Cost-Sensitive Online Multi-Label Classification}


\author{Hong-Min Chu         \and
        Kuan-Hao Huang       \and 
        Hsuan-Tien Lin        
}


\institute{Hong-Min Chu, Kuan-Hao Huang, Hsuan-Tien Lin \at
              CSIE Department, National Taiwan University, Taiwan \\
              \email{\{r04922031, r03922062, htlin\}@csie.ntu.edu.tw}           
}

\date{Received: date / Accepted: date}

\maketitle

\begin{abstract}
We study multi-label classification (MLC) with three important real-world issues: 
online updating, label space dimension reduction (LSDR), and cost-sensitivity.
Current MLC algorithms have not been designed to address these three issues simultaneously.
In this paper, we propose a novel algorithm, cost-sensitive dynamic principal projection (CS-DPP) that resolves all three issues.
The foundation of CS-DPP is an online LSDR framework derived from a leading LSDR algorithm. 
In particular, CS-DPP is equipped with an efficient online dimension reducer motivated by matrix stochastic gradient, 
and establishes its theoretical backbone when coupled with a carefully-designed online regression learner. 
In addition, CS-DPP embeds the cost information into label weights to achieve cost-sensitivity along with theoretical guarantees. 
Experimental results verify that CS-DPP achieves better practical performance than current MLC algorithms across different evaluation criteria, and demonstrate the importance of resolving the three issues simultaneously.
\end{abstract}

\section{Introduction}\label{sec:intro}

The multi-label classification (MLC) problem allows each instance to be associated with a set of labels and
reflects the nature of a wide spectrum of real-world applications \cite{nus-wide,emotions,yeast}.
Traditional MLC algorithms mainly tackle the batch MLC problem, where the input data are presented in a batch~\cite{cc,intro}. 
Nevertheless, in many MLC applications such as e-mail categorization \cite{stream_mlc_mtt}, multi-label examples arrive as a stream. 
Online analysis is therefore required
as batch MLC algorithms may not meet the needs to make a prediction and update the predictor on the fly.
The needs of such applications can be formalized as the online MLC (OMLC) problem. 

The OMLC problem is generally more challenging than the batch one, and many mature algorithms for the batch problem have not yet been carefully extended to OMLC.
Label space dimension reduction (LSDR) is a family of mature algorithms for the batch MLC problem~\cite{cplst,cs,faie,plst,bcs,cca,leml,cssp,landmark,sleec}. 
By viewing the label set of each instance as a high-dimensional label vector in a label space,
LSDR encodes each label vector as a code vector in  a lower-dimensional code space, and learns a predictor within the code space. 
An unseen instance is predicted by coupling the predictor with a decoder from the code space to the label space.
For example, compressed sensing (CS)~\cite{cs} encodes using random projections, and decodes with sparse vector reconstruction; 
principal label space transformation (PLST)~\cite{plst} encodes by projecting to the key eigenvectors of the known label vectors obtained from principal component analysis (PCA), and decodes by reconstruction with the same eigenvectors.
This low-dimensional encoding allows LSDR algorithms to exploit the key joint information between labels to be more robust to noise and be more effective on learning~\cite{plst}.
Nevertheless, to the best of our knowledge, all the LSDR algorithms mentioned above are designed only for the batch MLC problem. 

Another family of MLC algorithms that have not been carefully extended for OMLC contains the cost-sensitive MLC algorithms.
In particular, different MLC applications usually come with different evaluation criteria (costs) that reflect their realistic needs.
It is important to design MLC algorithms that are cost-sensitive to systematically cope with different costs, because
an MLC algorithm that targets one specific cost may not always perform well under other costs~\cite{cft}.
Two representative cost-sensitive MLC algorithms are probabilistic classifier chain (PCC)~\cite{pcc} and condensed filter tree (CFT)~\cite{cft}. 
PCC estimates the conditional probability with the classifier chain (CC) method~\cite{cc} and makes Bayes-optimal predictions with respect to the given cost; 
CFT decomposes the cost into instance weights when training the classifiers in CC. 
Both algorithms, again, targets the batch MLC problem rather than the OMLC one.

From the discussions above, there is currently no algorithm that considers the three realistic needs of online updating, label space dimension reduction, and cost-sensitivity at the same time. The goal of this work is to study such algorithms. 
We first formalize the OMLC and cost-sensitive OMLC (CSOMLC) problems in Section~\ref{sec:related} and discuss related work. 
We then extend LSDR for the OMLC problem and propose a novel online LSDR algorithm, dynamic principal projection (DPP), by connecting PLST with online PCA. 
In particular, we derive the DPP algorithm in Section~\ref{sec:DPP} along with its theoretical guarantees, and resolve the issue of possible basis drifting caused by online PCA. 

In Section~\ref{sec:CSDPP}, we further generalize DPP to cost-sensitive DPP (CS-DPP) to fully match the needs of CSOMLC with a theoretically-backed label-weighting scheme inspired by CFT. 
Extensive empirical studies demonstrate the strength of CS-DPP in addressing the three realistic needs in Section~\ref{sec:exp}. 
In particular, we justify the necessity to consider LSDR, basis drifting and cost-sensitivity. 
The results show that CS-DPP significantly outperforms other OMLC competitors across different CSOMLC problems, which validates the robustness and effectiveness of CS-DPP, as concluded in Section~\ref{sec:conclusion}.

\section{Preliminaries and Related Work}\label{sec:related} 

For the MLC problem, we denote the feature vector of an instance as $\Mv x \in \mathbb{R}^d $ and 
its corresponding label vector as $\Mv y \in \mathcal{Y} \equiv \{+1,-1\}^K$, where $\Mv y[k] = +1$ iff the instance is associated with the $k$-th label out of a total of $K$ possible labels.
We let $\Mv y[k] \in \{+1,-1\}$ to conform with the common setting of online binary classification~\cite{PA}, which is equivalent to another scheme, $\Mv y[k] \in \{1,0\}$, used in other MLC works \cite{cft,cc}.

Traditional MLC methods consider the batch setting, where a training dataset $ \mathcal{D} = \{(\Mv x_n, \Mv y_n)\}_{n=1}^{N}$ is given at once,
and the objective is to learn a classifier $g\colon \mathbb{R}^d \rightarrow \{+1,-1\}^K$ from $\mathcal{D}$
with the hope that $\Mvh y = g(\Mv x)$ accurately predicts the ground truth $\Mv y$ with respect to an unseen $\Mv x$.
In this work, we focus on the OMLC setting, which assumes that 
instance $(\Mv x_t, \Mv y_t)$ arrives in sequence from a data stream.
Whenever an~$\Mv x_t$ arrives at iteration $t$, the OMLC algorithm is required to make a prediction $\Mvh y_t = g_t(\Mv x_t)$ based on the current classifier~$g_t$ and feature vector $\Mv x_t$. 
The ground truth~$\Mv y_t$ with respect to $\Mv x_t$ is then revealed, and the penalty of $\Mvh y_t$ is evaluated against~$\Mv y_t$.

Many evaluation criteria for comparing $\Mv y$ and $\Mvh y$ have been considered in the literature to satisfy different application needs.
A simple criterion~\cite{intro} is the Hamming loss $c_{\textsc{ham}}(\Mv y, \Mvh y)$ = $\frac{1}{K}\sum^{K}_{k=1} \llbracket \Mv y[k] \neq \Mvh y[k] \rrbracket$.
The Hamming loss separately considers each label during evaluation. 
There are other criteria that jointly evaluate all labels, 
such as the F1 loss~~\cite{intro} 
$$c_{\textsc{f}}(\Mv y, \Mvh y) = 1 - 2 \dfrac{\sum\nolimits_{k=1}^K \bfunc{\Mv y[k]\!=\!+1 \mbox{ and } \Mvh y[k] \!=\!+1}}{\sum\nolimits_{k=1}^K \left(\bfunc{\Mv y[k]\!=\!+1} + \bfunc{\Mvh y[k]\!=\!+1}\right)}.$$
In this work, we follow existing cost-sensitive MLC approaches~\cite{cft} to extend OMLC to the cost-sensitive OMLC (CSOMLC) setting, which further takes
the evaluation criterion as an additional input to the learning algorithm. We call the criterion a \textit{cost function} and overload $c\colon \{+1, -1\}^K\times\{+1,-1\}^K \rightarrow \mathbb{R}$ as its notation. The cost function
evaluates the penalty of~$\Mvh y$ against~$\Mv y$ by $c(\Mv y, \Mvh y)$.
We naturally assume that $c(\cdot, \cdot)$ satisfies $c(\Mv y, \Mv y) = 0$ and $\max_{\Mvh y} c(\Mv y, \Mvh y) \leq 1$.
The objective of a CSOMLC algorithm is to adaptively learn a classifier $g_t\colon \mathbb{R}^d \rightarrow \{+1,-1\}^K$ 
based on not only the data stream but also the input cost function $c$ such that
the cumulative cost $\sum_{t=1}^T c(\Mv y_t, \Mvh y_t)$ with respect to the input $c$, where $\Mvh y_t = g_t(\Mv x_t)$, can be minimized. 

\journalminor{
  Note that the cost function within the CSOMLC setting above corresponds to the
  \emph{example-based} evaluation criteria for MLC, named because the prediction $\Mvh y_t$ of each example is evaluated
  against the ground truth $\Mv y_t$ independently. More sophisticated evaluation criteria such as micro-based
  and macro-based criteria~\cite{eval,Mao13} can also be found in the literature. The following equations highlight the difference between example-F1 (what our CSOMLC setting can handle), micro-F1 and macro-F1 when calculated on $T$ predictions
  \begin{eqnarray*}
    \mbox{example-F1 loss} &=& 1 - \frac{2}{T} \sum_{t=1}^T \frac{\sum\nolimits_{k=1}^K \bfunc{\Mv y_t[k]\!=\!+1 \mbox{ and } \Mvh y_t[k] \!=\!+1}}{\sum\nolimits_{k=1}^K \left(\bfunc{\Mv y_t[k]\!=\!+1} + \bfunc{\Mvh y_t[k]\!=\!+1}\right)} \;\;\; \journallast{;}\\
    \mbox{micro-F1 loss} &=& 1 - \frac{2}{K} \sum_{k=1}^K \frac{\sum\nolimits_{t=1}^T \bfunc{\Mv y_t[k]\!=\!+1 \mbox{ and } \Mvh y_t[k] \!=\!+1}}{\sum\nolimits_{t=1}^T \left(\bfunc{\Mv y_t[k]\!=\!+1} + \bfunc{\Mvh y_t[k]\!=\!+1}\right)} \;\;\; \journallast{;} \\
    \mbox{macro-F1 loss} &=& 1 - 2  \frac{\sum\nolimits_{t=1}^T \sum_{k=1}^K \bfunc{\Mv y_t[k]\!=\!+1 \mbox{ and } \Mvh y_t[k] \!=\!+1}}{\sum\nolimits_{t=1}^T \sum_{k=1}^K \left(\bfunc{\Mv y_t[k]\!=\!+1} + \bfunc{\Mvh y_t[k]\!=\!+1}\right)}  \;\;\; \journallast{.}
  \end{eqnarray*}
  In particular, the three criteria differ by the averaging process. Average example-F1 computes the geometric mean of precision and recall (F1) \textit{per example} and then computes the arithmetic mean over all examples; micro-F1 computes the geometic mean of precision and recall \textit{per label} and then computes the arithmetic mean over all labels; macro-F1 computes the geometric mean of precision and recall \textit{over the set of all example-label predictions}.
  The more sophisticated ones \journallast{are known to be} more difficult to optimize. 
  Thus, similar to many existing cost-sensitive MLC algorithms for the batch setting~\cite{cft}, 
  we consider only example-based criteria in this work,
  \journallast{and leave the investigation of achieving} cost-sensitivity for micro- and macro-based criteria \journallast{to the future}.
}

Several OMLC algorithms have been studied in the literature, including
online binary relevance~\cite{stream_mlc}, Bayesian OMLC framework~\cite{bayes_stream_mlc}, and the multi-window approach using $k$ nearest neighbors~\cite{cdp_stream_mlc}. However, none of them are cost-sensitive. That is, they cannot take the cost function into account to improve learning performance.

Cost-sensitive MLC algorithms have also been studied in the literature.
Cost-sensitive RA$k$EL~\cite{cs-rakel} and progressive RA$k$EL~\cite{prakel} are two algorithms
that generalize a famous batch MLC algorithm called RA$k$EL~\cite{rakel} to cost-sensitive learning.
The former achieves cost-sensitivity for any weighted Hamming loss, and the latter achieves this for any cost function.
Probabilistic classifier chain (PCC)~\cite{pcc} and condensed filter tree (CFT)~\cite{cft} are two other algorithms
that generalizes another famous batch MLC algorithm called classifier chain (CC)~\cite{cc} to cost-sensitive learning.
PCC estimates the conditional probability of the label vector via CC, and makes a Bayes-optimal prediction with respect
to the cost function and the estimation. 
PCC in principal achieves cost-sensitivity for any cost function, but the prediction
can be time-consuming unless an efficient Bayes inference rule is designed for the cost function (\textit{e.g.} the F1 loss~\cite{pcc_f1}). 
CFT embeds the cost information into CC by an $O(K^2)$-time step that re-weights the training instances for each classifier. All four algorithms above are designed for the batch cost-sensitive MLC problem, and it is not clear how they can be modified for the CSOMLC problem.
CC-family algorithms typically suffer from the problem of ordering the labels properly to achieve decent performance.
Some works start solving the ordering problem for the original CC algorithm, such as \journalminor{
  the easy-to-hard paradigm~\cite{Liu17},
} but
whether those works can be well-coupled with CFT or PCC \journallast{has yet} to be studied.


Label space dimension reduction (LSDR) is another family of MLC algorithms. 
LSDR encodes each label vector as a code vector in the lower-dimensional 
code space, and learns a predictor from the feature vectors to the corresponding code vectors.
The prediction of LSDR consists of the predictor followed by a decoder from the code space to the label space.
For example, compressed sensing (CS)~\cite{cs} uses random projection for encoding, takes a regressor as the predictor, and
decodes by sparse vector reconstruction. Instead of random projection,
principal label space transformation (PLST)~\cite{plst} encodes the label vectors
$\{\Mv y_n\}_{n=1}^N$ to their top principal components for the batch MLC problem.
Some other LSDR algorithms,
including conditional principal label space transformation (CPLST)~\cite{cplst}, 
feature-aware implicit label space encoding (FaIE)~\cite{faie}, 
canonical-correlation-analysis method~\cite{cca}, and low-rank empirical risk minimization for multi-label learning~\cite{leml},
jointly take the feature and the label vectors into account during encoding~\cite{cplst,faie,cca,leml} 
to further improve the performance.

  The physical intuition behind LSDR algorithms is to capture the key joint information between labels before learning. By encoding to a more concise code space, LSDR algorithms enjoy the advantage of learning the predictor more effectively to improve the MLC performance. Moreover, compared with non-LSDR algorithms like RA$k$EL and CFT, LSDR algorithms are generally more efficient, which in turn \journallast{makes them favorable} candidates to be extended to online learning.

Motivated by the possible applications of online updating, the realistic needs of cost-sensitivity, and the potential effectiveness of label space dimension reduction, we take an initiative to study LSDR algorithms for the CSOMLC setting. 
In particular, we first adapt PLST to the OMLC setting in Section~\ref{sec:DPP}, and further generalize it to the CSOMLC setting in Section~\ref{sec:CSDPP}.




    \section{Dynamic Principal Projection} \label{sec:DPP}
    \begin{table}[t]
\caption{Summary of common notations}
\label{tbl:notation}  
\centering
\resizebox{0.95\textwidth}{!}{
\begin{tabular}{ll}
notation & meaning \\ \hline \hline
$d$ & number of features \\
$K$ & number of labels \\
$M$ & dimension of the code space\\
$\Mv x \in \mathbb{R}^d$ & feature vector \\
$\Mv y \in \{+1, -1\}^K$ & ground truth label vector \\
$\Mvh y \in \{+1, -1\}^K$ & predicted label vector\\
$\Mv c(\Mv y, \Mvh y)$ & cost for predicting $\Mv y$ as $\Mvh y$\\
$\Mv z \in \mathbb{R}^M$ & code vector \\
$\Mv P \in \mathbb{R}^{M \times K}$ & encoding matrix from the label space to the code space\\
$\Mv W \in \mathbb{R}^{d \times M}$ & linear predictor matrix from the input space to the code space\\
$\Mv U \in \mathbb{R}^{K \times K}$ & (roughly) rank-$M$ matrix within matrix stochastic gradient (MSG)\\
$(\Mv Q \in \mathbb{R}^{(M+1) \times K}, \Mv \sigma \in \mathbb{R}^{M+1})$ & decomposition of $\Mv U$ such that $\Mv U = \Mv Q \mbox{diag}(\sigma) \Mv Q^\top$\\
$\Gamma \in [0, 1]^{M+1}$ & discrete probability distribution for sampling the rows of $\Mv Q$ to get $\Mv P$\\
$\delta^{(k)} \in \mathbb{R}$ & weight of the $k$-th label for representing the cost in CS-DPP\\
$\Mv C \in \mathbb{R}^{K \times K}$ & a diagonal matrix that stores $\{\sqrt{\delta^{(k)}}\}_{k=1}^K$ in CS-DPP
\end{tabular}
}
\end{table}

    In this section, we first propose an online LSDR algorithm, dynamic principal projection (DPP), that optimizes the Hamming loss.
    DPP is motivated by the connection between PLST, which encodes the label vectors to their top principal components,
    and the rich literature of online PCA algorithms~\cite{online_pca_sgd,online_pca_optimal_regret,online_pca_spca}.
    We shall first introduce the detail of PLST.
    Then, we discuss the potential difficulties along with our solutions to advance PLST to our proposed DPP.
      To facilitate reading, the common notations that will be used for the coming sections are summarized in Table~\ref{tbl:notation}.

    \subsection{Principal Label Space Transformation} \label{sec:PLST}

    Given the dimension $M \leq K$ of the code space and a batch training dataset
    $\mathcal{D} = \{(\Mv x_n, \Mv y_n)\}_{n=1}^N$, PLST, as a batch LSDR algorithm, encodes each 
    $\Mv y_n \in \{+1, -1\}^K$ into a code vector $\Mv z_n = \Mv P^* (\Mv y_n - \Mv o)$, where $\Mv o$ is a fixed reference point for shifting~$\Mv y_n$, and 
    $\Mv P^*$ contains the top~$M$ eigenvectors of $\sum_{n=1}^N (\Mv y_n - \Mv o) (\Mv y_n - \Mv o)^\top$. While PLST works with any fixed~$\Mv o$, it is worth noting that when $\Mv o$ is taken as $\frac{1}{N} \sum_{n=1}^N \Mv y_n$, the code vector~$\Mv z_n$ contains the top $M$ principal components of $\Mv y_n$.
    A multi-target regressor $\Mv r$ is then learned on $\{(\Mv x_n, \Mv z_n)\}_{n=1}^N$, 
    and the prediction of an unseen instance~$\Mv x$ is made by
    \begin{equation}\label{eq:PLST_predict}
    \Mvh y = \mathrm{round}\left((\Mv P^*)^\top \Mv r(\Mv x)  + \Mv o\right) 
    \end{equation}
    where%
    \footnote{\journalminor{The naming of the $\mathrm{round}(\cdot)$ operator follows directly from the original paper of PLST~\cite{plst}, which represents $\Mv y \in \{0, 1\}^K$ instead of $\{-1, +1\}^K$. Our use of $\mathrm{sign}$ is thus equivalent to the rounding steps used in the original PLST.}}
    $
    \mathrm{round}(\Mv v) = 
    \bigl( \mathrm{sign}(\Mv v[1]), \hdots, \mathrm{sign}(\Mv v[K]) \bigr)^\top$. 

    By projecting to the top principal components, PLST preserves the maximum amount of information within the observed label vectors. In addition, PLST is backed by the following theoretical guarantee:
    \begin{theorem}\cite{plst}\label{thm:hamming_bound}
    When making a prediction~$\Mvh y$ from $\Mv x$ by $\Mvh y = \mathrm{round}\left(\Mv P^\top \Mv r(\Mv x) + \Mv o\right)$ with any left orthogonal matrix $\Mv P$, the Hamming loss
    \begin{equation}\label{eq:hamming_bound}
    c_{\textsc{ham}}(\Mv y, \Mvh y) \leq \frac{1}{K} (\underbrace{\|\Mv r(\Mv x) - \Mv z\|^2_2}_{\text{pred. error}} + 
    \underbrace{\|(\Mv I - \Mv P^\top\Mv P)(\Mv y')\|^2_2}_{\text{reconstruction error}})
    \end{equation}
    where $\Mv z \equiv \Mv P \Mv y'$ and $\Mv y' \equiv \Mv y - \Mv o$ with respect to any fixed reference point $\Mv o$.
    \end{theorem}
    Theorem~\ref{thm:hamming_bound} bounds the Hamming loss by the prediction and reconstruction errors. Based on the results of singular value decomposition, $\Mv P^*$ in PLST is the optimal solution for minimizing the total reconstruction error of the observed label vectors with respect to any fixed $\Mv o$, and the particular reference point $\frac{1}{N} \sum_{n=1}^N \Mv y_n$ minimizes the reconstruction error over all possible $\Mv o$.
    Then, by minimizing the prediction error with regressor $\Mv r$, PLST is able to minimize the Hamming loss approximately.


    \subsection{General Online LSDR Framework for DPP}

    The upper bound in Theorem~\ref{thm:hamming_bound} works for any regressor $\Mv r$ and any left orthogonal encoding matrix $\Mv P$. Based on the bound, we propose an online LSDR framework that approximately minimizes the Hamming loss with an online regressor $\Mv r_t$ and an online encoding matrix $\Mv P_t$ in each iteration $t$. Similar to PLST, the proposed framework works with any fixed referenced point $\Mv o$. But for simplicity of illustration, we assume that $\Mv o = \Mv 0$ to remove $\Mv o$ from the derivations below. The steps of the framework are:

    \vspace*{.5\baselineskip}
    \fbox{\begin{minipage}[c]{.6\textwidth}
    \noindent For $t = 1,\hdots,T$ \\
    \hspace*{1em} Receive $\Mv x_t$ and predict $\Mvh y_t = \mathrm{round}(\Mv P^\top_t \Mv r_t(\Mv x_t))$ \\
    \hspace*{1em} Receive $\Mv y_t$ and incur error $\ell^{(t)}(\Mv r_t, \Mv P_t)$  \\
    \hspace*{1em} Update $\Mv P_t$ and $\Mv r_t$ 
    \end{minipage}
    }

    \vspace*{.5\baselineskip}
    \noindent In each iteration $t$ of the framework, an online prediction $\Mvh y_t$ is made with the updated $\Mv r_t$ and $\Mv P_t$. We take the online error function $\ell^{(t)}(\Mv r, \Mv P)$ to be $\|\Mv r(\Mv x_t) - \Mv P \Mv y_t\|_2^2 + \|(\Mv I - \Mv P^\top \Mv P)\Mv y_t\|_2^2$, which upper bounds the Hamming loss $c_{\textsc{ham}}(\Mv y_t, \Mvh y_t)$ of the online prediction. Then, by updating $\Mv r_t$ and $\Mv P_t$ with online learning algorithms that minimize the cumulative online error $\sum_{t=1}^T \ell^{(t)}(\Mv r_t, \Mv P_t)$, we can approximately minimize the cumulative Hamming loss.

    The simple framework above transforms the OMLC problem to an online learning problem with an error function composed of two terms. 
      Ideally, the online learning algorithm should update $\Mv P_t$ and $\Mv r_t$ to jointly minimize the total error from both terms.
      Optimizing the two terms jointly has been studied in batch LSDR algorithms like CPLST~\cite{cplst}, which is a successor of PLST~\cite{plst} that
      also operates with the upper bound in Theorem~\ref{thm:hamming_bound}.
      Nevertheless, it is very challenging to extend CPLST to the online setting efficiently. In particular, a
      na{\i}ve online extension would require computing the hat matrix of the ridge regression part (from $\Mv x$ to $\Mv z$) within CPLST in order
      to obtain $\Mv P_t$, 
      and the hat matrix grows quadratically with the number of examples. That is, in an online setting, computing and storing the hat matrix needs at
      least $\Omega(T^2)$ complexity up to iteration $T$, which is practically infeasible.

      Thus, we resort to PLST~\cite{plst}, the predecessor of CPLST, to make an initial attempt towards tackling OMLC problems. PLST minimizes the two terms separately in the batch setting, and our proposed extension of PLST similarly contains two online learning algorithms, one for minimizing each term.
      That is, we further decompose the online learning problem to two sub-problems, one for minimizing the cumulative reconstruction error (by updating~$\Mv P_t$), and one for minimizing the cumulative prediction error (by updating $\Mv r_t$). 
      Designing efficient and effective algorithms for the two sub-problems turns out to be non-trivial,
      and will be discussed
    in Sections~\ref{sec:P_learning} and \ref{sec:W_learning}.


    \subsection{Online Minimization of Reconstruction Error} \label{sec:P_learning}

    Next, we discuss the design of our first online learning algorithm to tackle the sub-problem of minimizing the cumulative reconstruction error  
    $\sum_{t=1}^T \|(\Mv I - \Mv P_t^\top \Mv P_t)\Mv y_t\|_2^2$, which corresponds to the second term in (\ref{eq:hamming_bound}). The goal is to generate a left-orthogonal matrix $\Mv P_t \in \mathbb{R}^{M \times K}$ in each iteration that guarantees to minimize the cumulative reconstruction error theoretically.

    Our design is motivated by a simple but promising online PCA algorithm,  
    matrix stochastic gradient (MSG)~\cite{online_pca_sgd}. 
    MSG does not directly solve the sub-problem of our interest because the problem is non-convex over $\Mv P_t$. Instead,
    MSG substitutes $\Mv P_t^\top \Mv P_t$ with a rank-$M$ matrix $\Mv U_t \in \mathbb{R}^{K \times K}$ and rewrites the cumulative reconstruction error as $\sum_{t=1}^T \Mv y_t^\top (\Mv I - \Mv U_t) \Mv y_t$.
    By further assuming that $\|\Mv y_t\|_2 \leq 1$, MSG loosens the constraint of $rank(\Mv U_t) = M$ to $tr(\Mv U_t) = M$, 
and updates $\Mv U_t$ with online projected gradient descent upon receiving a new~$\Mv y_t$ as
\begin{equation}\label{eq:MSG update}
\begin{aligned}
\Mv U_{t+1} = \mathcal{P}_{tr}(\Mv U_t + \eta \Mv y_t \Mv y_t^\top)
\end{aligned}
\end{equation}
where $\eta$ is the learning rate and $\mathcal{P}_{tr}(\cdot)$ is the projecting operator to a feasible $\Mv U$. The less-constrained $\Mv U_t$ in MSG carries the theoretical guarantee of minimizing the cumulative reconstruction error (subject to $\Mv U_t$), but decomposing $\Mv U_t$ to a left-orthogonal $\Mv P_t \in \mathbb{R}^{M \times K}$ with theoretical guarantee on~$\Mv P_t$ is not only non-trivial but also time-consuming.

Capped MSG~\cite{online_pca_sgd} is an extension of MSG with the \journallast{hope of lightening the computational} burden of decomposing $\Mv U_t$.
In particular, Capped MSG introduces an additional (non-convex) constraint of $rank(\Mv U_t) \leq M + 1$,
and indirectly maintains the decomposition of $\Mv U_t$ 
as $(\Mv Q_t, \Mv \sigma_t)$, where
the left-orthogonal matrix $\Mv Q_t \in \mathbb{R}^{(M+1)\times K}$ and the vector of singular values $\sigma_t \in \mathbb{R}^{M+1}$ 
such that $\Mv U_t = \Mv Q_t \mbox{diag}(\sigma_t) \Mv Q_t^\top$.
The decomposed $(\Mv Q_t, \Mv \sigma_t)$ in Capped MSG enjoys the same theoretical guarantee of minimizing the reconstruction error as the $\Mv U_t$ in MSG,
while the maintenance step of Capped MSG is more efficient than MSG.
Nevertheless, because we want $\Mv P_t$ to be $M$ by $K$ while $\Mv Q_t$ is $(M\!+\!1)$ by $K$,
the generated $\Mv Q_t$ in Capped MSG cannot be directly used to solve our sub-problem.
A na{\"i}ve idea is to generate $\Mv P_t$ by truncating the least important row of $\Mv Q_t$, but the na{\"i}ve idea is no longer backed by the theoretical guarantee of Capped MSG.


Aiming to address the above difficulties, we propose an efficient and effective algorithm to stochastically 
generate $\Mv P_t$ from $(\Mv Q_t, \sigma_t)$ maintained by Capped MSG in each iteration. 
To elaborate, let $\Mv Q_t^{-i}$ be $\Mv Q_t$ with its $i$-th row removed and $\sigma_t[i]$ be the eigenvalue corresponding to $i$-th row of $\Mv Q_t$.
We generate $\Mv P_t$ by sampling from a discrete probability distribution $\Gamma_t$,
which consists of $M+1$ events  $\{\Mv Q_t^{-i}\}_{i=1}^{M+1}$ 
with probability of $\Mv Q_t^{-i}$ being $1-\sigma_t[i]$.
As the projecting operator $\mathcal{P}_{tr}(\cdot)$ ensures $0 \leq \sigma_t[i] \leq 1$ for each $\sigma_t[i]$,
one can easily verify $\Gamma_t$ to be a valid distribution with the additional fact that $\sum_{i}\sigma_t[i] = tr(\Mv U_t) = M$.
The following lemma shows that the online encoding matrix generated by our simple stochastic algorithm is truly effective, 
and the proof can be found in the supplementary materials.
\begin{lemma}\label{lem:P sampling}
Suppose $(\Mv Q_t, \sigma_t)$ is obtained after an updated of Capped MSG such that $\Mv U_t = \Mv Q_t diag(\sigma_t) \Mv Q_t^\top$.
If $\Gamma_t$ is a discrete probability distribution over events $\{\Mv Q_t^{-i}\}_{i=1}^{M+1}$ with probability of $\Mv Q_t^{-i}$ being $1-\sigma_t[i]$,
we have for any $\Mv y$ 
\begin{equation}\label{eq:P sampling}
\mathbb{E}_{\Mv P_t \sim \Gamma_t} [\Mv y^\top (\Mv I - \Mv P_t^\top \Mv P_t) \Mv y ] = \Mv y^\top (\Mv I - \Mv U_t) \Mv y
\end{equation}
\end{lemma}
The proof of the lemma can be found in \ref{proof:P sampling}.
Moreover, our sampling algorithm is highly efficient regarding its $\mathcal{O}(M)$ time complexity.
\journallast{
  Note that there is an earlier work that contains another algorithm of similar spirit~\cite{online_pca_optimal_regret}. Somehow the algorithm's time complexity is
  $\mathcal{O}(K^2)$, which is less efficient than ours.
  }

To sum up, our online learning algorithm that minimizes the cumulative reconstruction error for DPP takes Capped MSG as its building block
to maintain $\Mv U_t$ by $\Mv Q_t$ and $\sigma_t$, and then samples the online encoding matrix $\Mv P_t$ from $\Gamma_t$ derived by $\Mv Q_t$ in each iteration by our proposed sampling algorithm.
Note that to fulfill the assumption of $\|\Mv y_t\|_2 \leq 1$ required by Capped MSG,
we apply a simple trick to scale each $\Mv y_t \in \{+1,-1\}^K$ with a factor of $\frac{1}{\sqrt{K}}$.
The predictions given by our online LSDR framework remain unchanged after the constant scaling due to the use of $\mathrm{round}(\cdot)$ operator. 

\subsection{Online Minimization of Prediction Error}\label{sec:W_learning}
Next, we discuss another proposed online learning algorithm to 
solve the second sub-problem of minimizing the cumulative prediction error $\sum_{t=1}^T \|\Mv r_t(\Mv x_t) - \Mv P_t \Mv y_t\|^2$, 
which corresponds to the first term in (\ref{eq:hamming_bound}).
The proposed online learning algorithm is based on the well-known online ridge regression, 
and incorporates two different carefully designed techniques to remedy the negative effect caused by the variation of $\Mv P_t$ in each iteration. 

The na{\"i}ve online ridge regression parameterizes  $\Mv r_t(\Mv x)$ 
to be an online linear regressor $\Mv W_t^\top \Mv x$ with $\Mv W_t \in \mathbb{R}^{d \times M}$,
and update $\Mv W_t$ by 
\begin{equation}\label{eq:W_naive}
\Mv W_t = \argmin{\Mv W} \frac{\lambda}{2} tr(\Mv W \Mv W^\top) + \sum_{i=1}^{t-1} \|\Mv W^\top \Mv x_i - \Mv z_i \|_2^2 
\end{equation}
where $\Mv z_i = \Mv P_i \Mv y_i$ is the code vector of $\Mv y_i$ regarding $\Mv P_i$, and $\lambda$ is the regularization parameter.
However, the na{\"i}ve online ridge regression suffers from the drifting of projection basis 
caused by \journallast{varying the online} encoding matrix $\Mv P_t$ as $t$ advances. 
To elaborate, recall that the online regressor $\Mv W_t$ aims to predict $\Mv z_t = \Mv P_t \Mv y_t$ from $\Mv x_t$,
where the code vector $\Mv z_t$ can essentially be viewed as the set of combination coefficients with reference projection basis formed by $\Mv P_t$. 
However, $\Mv W_t$ is learned from $\{(\Mv x_i, \Mv z_i)\}_{i=1}^{t-1}$, 
where the learning target $\{\Mv z_i\}_{i=1}^{t-1}$ is mixed up with coefficients $\Mv z_i$ induced from different projection basis $\Mv P_i$.
As a consequence, expecting $\Mv W_t^\top \Mv x_t$ to give accurate prediction of $\Mv z_t$ for any specific $\Mv P_t$ is unrealistic.
  For a very extreme case, if $\Mv P_1 = \Mv P_3 = \ldots = \Mv P_{2\tau-1} = \Mv P$ and $\Mv P_2 = \Mv P_4 = \ldots = \Mv P_{2\tau} = - \Mv P$, the $\Mv z_i$'s in the odd and even iterations are of totally opposite meanings although the projection matrices $\Mv P$ and $-\Mv P$ are mathematically equivalent in quality. The totally opposite meanings make it impossible for $\Mv W_t$ to predict $\Mv z_t$ accurately.

To remedy the problem of basis drifting, we propose two different techniques, principal basis correction (PBC) and principal basis transform (PBT),
to improve online regressor $\Mv W_t$.
Each of them enjoys different advantages.

\subsubsection{Principal Basis Correction}
The ideal solution to handle basis drifting is to ``correct'' the reference basis of each $\Mv z_i$ to be the latest $\Mv P_t$ used for prediction. 
More specifically, we want $\Mv W_t$ to be the ridge regression solution obtained from $\{(\Mv x_i, \Mv P_t \Mv y_i)\}_{i=1}^{t-1}$
instead of $\{(\Mv x_i, \Mv P_i \Mv y_i)\}_{i=1}^{t-1}$.
  Such a correction step ensures that the reference basis \journallast{for generating the previous} $\Mv z_i$'s is the same as \journallast{the basis that will be} used for the predicting $\Mv z_t$ and decoding $\Mvh y_t$ from $\Mv z_t$. Denote $\Mv W^{\text{PBC}}_t$ as the  ridge regression solution of $\{(\Mv x_i, \Mv P_t \Mv y_i)\}_{i=1}^{t-1}$. The closed-form solution of $\Mv W^{\text{PBC}}_t$ is
\begin{equation}\label{eq:W_PBC}
\Mv W^{\text{PBC}}_t = \underbrace{(\lambda \Mv I + \sum_{i=1}^{t-1} \Mv x_i \Mv x_i^\top)^{-1}}_{\Mv A_{t}^{-1}}\underbrace{(\sum_{i=1}^{t-1} \Mv x_i \Mv y_i^\top)}_{\Mv B_{t}} \Mv P_t^\top \; \journallast{.} 
\end{equation}
  The part $\Mv A_t^{-1} \Mv B_t$ is independent of the projection matrix $\Mv P_t$. Thus,
  by maintaining another $d$ by $K$ matrix \begin{eqnarray*}
    \Mv H_t = \Mv A_t^{-1} \Mv B_t \label{eq:ht}
  \end{eqnarray*}
    throughout the iterations,
  $\Mv W^{\text{PBC}}_t$ can be easily obtained by $\Mv H_t \Mv P_t^\top$ for any~$\Mv P_t$.
\journalminor{
  The update of $\Mv H_t$ to $\Mv H_{t+1}$, on the other hand, 
  requires the calculation of $\Mv H_{t+1} = (\Mv A_t + \Mv x_t \Mv x_t^\top)^{-1}(\Mv B_t + \Mv x_t \Mv y_t^\top)$,
  which at a first glance \journallast{has a time complexity of $\mathcal{O}(d^3 + Kd^2)$.}
  Fortunately, we can speed up the calculation by applying the Sherman-Morrison formula, which states that
  \begin{eqnarray*}
    (\Mv A_t + \Mv x_t \Mv x_t^\top)^{-1} &=& \left(\Mv A_t^{-1} - \frac{\Mv A_t^{-1} \Mv x_t \Mv x_t^\top \Mv A_t^{-1}}{1 + \gamma}\right)
  \end{eqnarray*}
  with $\gamma = \Mv x_t^\top \Mv A_t^{-1} \Mv x_t$.
  Then, the calculation can be rewritten as
  \begin{align*}
    \Mv H_{t+1} & = \left(\Mv A_t^{-1} - \frac{\Mv A_t^{-1} \Mv x_t \Mv x_t^\top \Mv A_t^{-1}}{1 + \gamma}\right)\left(\Mv B_t + \Mv x_t \Mv y_t^\top\right)  \\
    & = \Mv A_t^{-1} \Mv B_t - \frac{\Mv A_t^{-1} \Mv x_t \Mv x_t^\top \Mv A_t^{-1} \Mv B_t}{1 + \gamma} + \Mv A_t^{-1} \Mv x_t \Mv y_t^\top  - \frac{\Mv A_t^{-1} \Mv x_t \Mv x_t^\top \Mv A_t^{-1} \Mv x_t \Mv y_t^\top}{1 + \gamma} \\
              & = \Mv H_t - \frac{\Mv A_t^{-1}\Mv x_t \Mvt y_t^\top}{1 + \gamma} + \Mv A_t^{-1} \Mv x_t \Mv y_t^\top - \frac{\gamma \Mv A_t^{-1} \Mv x_t \Mv y_t^\top}{1+\gamma}\\
              & = \Mv H_{t} - \frac{\Mv A_{t}^{-1} \Mv x_{t} (\Mvt y_{t} - \Mv y_{t})^\top}{1 + \gamma}, \\
  \end{align*}
  where $\Mvt y_t = \Mv H_t^\top \Mv x_t$. The third line follows from the fact that $\Mv H_t = \Mv A_t^{-1} \Mv B_t$. 
Thus, the $d$ by~$K$ matrix $\Mv H_t$ can be efficiently updated online by
\begin{equation}\label{eq:W update}
\Mv H_{t+1} = \Mv H_{t} - \frac{\Mv A_{t}^{-1} \Mv x_{t} (\Mvt y_{t} - \Mv y_{t})^\top}{1 + \Mv x_{t}^\top \Mv A_{t}^{-1} \Mv x_{{t}}}
\end{equation}
which requires \journallast{only a time complexity of $\mathcal{O}(d^2 + Kd)$.}
}

\journalminor{
It is worth noting that $\Mv H_t$ actually stores the online ridge regression solution from $\Mv x$ to~$\Mv y$. Based on the definition
of $\Mv H_t$, we can then theoretically analyze the performance of our online ridge regression solution $\Mv W_t^{\textsc{PBC}}$ from $\Mv x$ to $\Mv z$ with respect to the error $\ell^{(t)}(\cdot,\cdot)$ in our proposed online LSDR framework.
}
Following the convention of online learning, we analyze the expected average regret $\frac{\mathcal{R}}{T}$, defined as
\begin{eqnarray}
  \frac{\mathcal{R}}{T} = \frac{1}{T}\sum_{t=1}^T\mathbb{E}_{\Mv P_t\sim \Gamma_t}[\ell^{(t)}(\Mv W_t^{\text{PBC}}, \Mv P_t) - \ell^{(t)}(\Mv W_\#, \Mv P^*)], \label{eq:regret}
\end{eqnarray}
for any given sequence of $\{(\Mv P_t, \Gamma_t)\}_{t=1}^T$, where each $\Mv P_t$ is sampled from the distribution~$\Gamma_t$. 
$(\Mv W_\#, \Mv P^\ast)$ here denotes the offline reference algorithm that is allowed to peek the whole data stream $\{(\Mv x_t, \Mv y_t)\}_{t=1}^T$. 
As our algorithm aims to minimize the online error function by a similar decomposition of sub-problems as PLST
, we particularly consider $(\Mv W_\#, \Mv P^\ast)$ to be the solution of PLST when treating $\{(\Mv x_t, \Mv y_t)\}_{t=1}^T$ as the input batch data.
That is,
$\Mv P^*$ is the minimizer of $\sum_{t=1}^T \Mv y_t^\top(\Mv I-\Mv P^\top \Mv P) \Mv y_t$, which corresponds to the second term of $\ell^{(t)}(\cdot,\cdot)$,
and $\Mv W_\#$  is the minimizer of $\sum_{t=1}^T\|\Mv W^\top \Mv x_t - \Mv P^* \Mv y_t\|_2^2$, which corresponds to the first term of $\ell^{(t)}(\cdot,\cdot)$ given $\Mv P^*$.
It can be easily proved that $\Mv W_\# = \Mv H^* (\Mv P^*)^\top$ where $\Mv H^*$ is the optimal linear regression solution of $\{(\Mv x_t, \Mv y_t)\}_{t=1}^T$. That is,
\begin{eqnarray}
\Mv H^* = \argmin{\Mv H} \sum_{t=1}^T \|\Mv H^\top \Mv x_t - \Mv y_t\|_2^2 \; . \label{eq:hstar}
\end{eqnarray}


\journalminor{
  With the expected average regret defined, we can prove its convergence by assuming the convergence of the subspace spanned by $\Mv P_t$ to the subspace spanned by $\Mv P^*$. The assumption generally holds when the $M$-th and $(M+1)$-th eigenvalues of $\sum_{t=1}^T (\Mv y_t - \Mv o) (\Mv y_t - \Mv o)^\top$ are different, as the subspace spanned by $\Mv P^*$ to reach the minimum reconstruction error is consequently unique. In particular, define the expected subspace difference
  \begin{eqnarray}
    \Delta_t = \|\mathbb{E}_{\Mv P_t \sim \Gamma_t}[\Mv P_t^\top \Mv P_t] - \left(\Mv P^*\right)^\top \Mv P^*\|_2 \; \journallast{.} \label{eq:subspace}
  \end{eqnarray}  
  \begin{theorem}\label{thm:DPP_drift} With the definitions of $\Mv H_t$ in \eqref{eq:ht}, $\Mv H^*$ in \eqref{eq:hstar}, $\frac{\mathcal{R}}{T}$ in \eqref{eq:regret} and $\Delta_t$ in \eqref{eq:subspace}, assume that $\|\Mv x_t\| \le 1$, $\|\Mv y_t\| \le 1$ and
    $\|\Mv H_t \Mv x_t - \Mv y_t\|_2^2 \le \epsilon$.
    \begin{enumerate}
    \item For any given $T$, the expected cumulative regret $\mathcal{R}$ is upper-bounded by
      \[
      (1+\epsilon) \sum_{t=1}^T \Delta_t + \frac{M}{2} \|\Mv H^*\|_F^2 + 2 \epsilon Md \log \left(1 + \frac{T}{d}\right).
      \]
    \item If $\lim_{T \rightarrow \infty} \Delta_T = 0$ and $\|\Mv H^*\|_F \le h^*$ across all iterations,\footnote{
      The technicality of requiring $\|\Mv H^*\|_F$ to be bounded is because we defined regret (up to the $T$-th iteration) with respect to the optimal offline solution upon receiving $T$ examples, and hence $\Mv H^*$ depends on $T$. Standard regret proof in online learning alternatively defines regret with respect to any \textit{fixed} $\Mv H$.  Our proof could also go through with the alternative definition, which changes $\|\Mv H^*\|_F$ to a constant $\|\Mv H\|_F$ (that is trivially bounded).
      }
      $\lim_{T \rightarrow \infty} \frac{\mathcal{R}}{T} = 0$.
    \end{enumerate}
  \end{theorem}
}
\journalminor{
The third assumption requires the residual errors of online ridge regression without projection to be bounded, which generally holds
when there is some linear relationship between $\Mv x_t$ and $\Mv y_t$.
}
The detailed of the proof of the theorem can be found in \ref{proof:DPP_drift}.
Theorem~\ref{thm:DPP_drift} guarantees the performance of PBC to be competitive 
with a reasonable offline baseline in the long run given the convergence of subspace spanned by $\Mv P_t$. 
\journallast{Such a guarantee} makes online linear regressor with PBC a solid option for DPP to tackle the sub-problem of minimizing cumulative prediction error. 



\subsubsection{Principal Basis Transform}
While PBC always gives the $\Mv W^{\text{PBC}}_t$ learned on the correct code vectors with respect to the basis formed by $\Mv P_t$, 
the time and space complexity of PBC depends on $\Omega(Kd)$ at the cost of maintaining $\Mv H_t \in \mathbb{R}^{d\times K}$. 
The $\Omega(Kd)$ dependency can make PBC computationally inefficient when both $K$ and $d$ are large. 

To address the issue, we propose another technique, principal basis transform (PBT). 
Different from PBC, 
when a new online encoding matrix $\Mv P_{t+1}$ is presented, 
PBT aims at a direct basis transform of the online linear regressor from $\Mv P_t$ to $\Mv P_{t+1}$. 
To be more specific, PBT assumes the regressor $\Mv W_t^{\textsc{PBT}}$ to be the 
low-rank coefficients matrix of some \emph{unknown} $\Mv H_t' \in \mathbb{R}^{d \times K}$
with reference projection basis formed by $\Mv P_t$, which can equivalently be described as $\Mv W_t^{\textsc{PBT}} = \Mv H'_t \Mv P_t^\top$. 
The goal of PBT is to update $\Mv W_t^{\textsc{PBT}}$ to  $\Mv W_{t+1}^{\textsc{PBT}}$ with $(\Mv x_t, \Mv y_t)$ 
such that the reference projection basis of $\Mv W_{t+1}^{\textsc{PBT}}$ is now induced from $\Mv P_{t+1}$.
PBT achieves the goal by a two-step procedure. 
The first step is to find the low-rank coefficients matrix $\Mv W_t'$ of $\Mv H_t'$ based on the new reference basis formed by $\Mv P_{t+1}$.
However, as only the low rank coefficients matrix $\Mv W_{t}^{\textsc{PBT}}$
rather than $\Mv H_t'$ itself is known, we approximate $\Mv W_t'$ by 
\begin{equation}\label{eq:min_transform}
\Mv W_t' = \argmin{\Mv W} \|\Mv W \Mv P_{t+1} -  \Mv W_{t}^{\textsc{PBT}} \Mv P_{t} \|_F^2  \; \journallast{.}
\end{equation}
Solving (\ref{eq:min_transform}) analytically gives 
\begin{equation}\label{eq:sol_min_transform}
\Mv W_t' = \Mv W_{t}^{\textsc{PBT}} \Mv P_{t} \Mv P_{t+1}^\top \; \journallast{.}
\end{equation}
The second step is to update $\Mv W_t'$ with $(\Mv x_t, \Mv y_t)$ to obtain $\Mv W_{t+1}^{\textsc{PBT}}$ by 
\begin{equation}\label{eq:W update PBT}
\Mv W_{t+1}^{\textsc{PBT}} = \Mv W_{t}' - \frac{\Mv A_{t}^{-1} \Mv x_{t} (\Mvt z'
        _{t} - \Mv P_{t+1} \Mv y_{t})^\top}{1 + \Mv x_{t}^\top \Mv A_{t}^{-1} \Mv x_{{t}}}
\end{equation}
where $\Mvt z'_t = \left(\Mv W_t'\right)^\top \Mv x_t$.
\journalminor{
Equation (\ref{eq:W update PBT}) can be derived with a similar use of the Sherman-Morrison formula as that for (\ref{eq:W update}) 
by replacing $(\Mvt y_t, \Mv y_t)$ with $(\Mvt z'_t, \Mv P_t \Mv y_t)$ respectively. 
}
One can easily verify that $\Mv W_{t+1}^{\textsc{PBT}}$ obtained by (\ref{eq:W update PBT}) still keeps its reference basis as $\Mv P_{t+1}$.

Comparing to PBC, PBT only has $\Omega(M^2(K+d))$ dependency, which is particularly useful when $M^2 \ll min(K,d)$.
The appealing time complexity makes PBT a highly practical option for DPP to minimize the cumulative prediction error with.
  The time and space complexity of the two variants of DPP are listed in Table~\ref{tbl:complexity}.
\begin{table}[t]
\centering
\caption{Time and space complexity for two DPP variants}
\label{tbl:complexity}
\resizebox{0.8\textwidth}{!}{
\begin{tabular}{l| c | c}
 & time complexity & space complexity \\ 
\hline
DPP-PBC & $\mathcal{O}(d^2 + MK + Kd + M^2K)$ & $\mathcal{O}(d^2 + MK + Kd)$ \\
\hline
DPP-PBT & $\mathcal{O}(d^2 + M^2d + M^2K)$ & $\mathcal{O}(d^2 + MK + Md)$ \\
\end{tabular}
}
\end{table}

\section{Generalization to Cost-Sensitive Learning}\label{sec:CSDPP}

In this section, we generalize DPP to cost-sensitive DPP (CS-DPP), which meets the requirement of CSOMLC. 
The key ingredient to the generalization is a carefully designed label-weighting scheme that transforms cost $c(\Mv y, \Mvh y)$
into the corresponding weighted Hamming loss. 
With the help of the label weighting scheme,
we subsequently derive the optimization objective similar to Theorem~\ref{thm:hamming_bound} for general cost functions,
which allows us to derive CS-DPP by reusing the building blocks of DPP.


\begin{algorithm}[t]
\caption{Cost-Sensitive Dynamic Principal Projection with Principal Basis Transform}
\label{alg:CSDPP}
\begin{algorithmic}[1]

\setre{Parameters:}
\REQUIRE $\lambda$, $\eta$, $M$

    \STATE $\Mv P_0 \leftarrow \Mv O_{M \times K}$, $\Mv U_0 \leftarrow \Mv O_{K \times K}$, $\Mv A_0^{-1} \leftarrow \frac{1}{\lambda} \Mv I_{d \times d}$, $\Mv W_0 \leftarrow \Mv O_{d \times M}$ ($\Mv O$ is zero matrix) 
	\WHILE{Receive $(\Mv x_t, \Mv y_t)$}
		\STATE $\Mvh y_t \leftarrow \mathrm{round}(\Mv P_{t-1}^\top \Mv W^\top_{t-1} \Mv x_t)$
        \STATE Obtain $\Mv C_t$ by (\ref{eq:delta})
        \STATE Update $\Mv U_{t-1}$ to $\Mv U_t$ by Capped MSG (with $\Mv C_t \Mv y_t$) and sample $\Mv P_t$ from $\Gamma_t$ as defined in Lemma~\ref{lem:P sampling} 
        \STATE $\Mv W'_{t-1} \leftarrow \Mv W_{t-1} \Mv P_{t-1} \Mv P_t^\top$ (PBT)
        \STATE Update $\Mv W'_{t-1}$, $\Mv A_{t-1}^{-1}$ to $\Mv W_t$, $\Mv A_{t}^{-1}$ by (\ref{eq:W update PBT}) (with $\Mv C_t \Mv y_t$)

	\ENDWHILE
\end{algorithmic}
\end{algorithm}


We start from the detail of our label-weighting scheme based on the label-wise decomposition of $c(\Mv y, \Mvh y)$. 
To represent the cost with the label weights,
we propose a label-weighting scheme based on a label-wise and \emph{order-dependent} decomposition of $c(\cdot,\cdot)$, 
which is motivated by a similar concept in \cite{cft}. 
The label-weighting scheme works as follows.
Defining $\Mvh y_{\text{real}}^{(k)}$ and $\Mvh y_{\text{pred}}^{(k)}$ as  
\[
\Mvh y_{\text{real}}^{(k)}[i] = 
\begin{cases}
    \Mv y[i] & \text{if} \, i < k \\
    \Mv y[i] & \text{if} \, i = k \\
    \Mvh y[i] & \text{if} \, i > k \\
\end{cases}
\: \text{\textbf{and}} \:
\Mvh y_{\text{pred}}^{(k)}[i] = 
\begin{cases}
    \Mv y[i] & \text{if} \, i < k \\
    -\Mv y[i] & \text{if} \, i = k \\
    \Mvh y[i] & \text{if} \, i > k \\
\end{cases}
\]
we decompose $c(\Mv y, \Mvh y)$ into $\delta^{(1)},\hdots,\delta^{(K)}$ such that  
\begin{equation}\label{eq:label_weight}
\delta^{(k)} = |c(\Mv y, \Mvh y_{\text{pred}}^{(k)}) - c(\Mv y, \Mvh y_{\text{real}}^{(k)})| \; \journallast{.}
\end{equation}
\journalminor{
Recall that $\Mv y$ is the ground truth vector and $\Mvh y$ is the prediction vector from the algorithm. The two newly constructed vectors, $\Mvh y^{(k)}_{\text{real}}$ and $\Mvh y^{(k)}_{\text{pred}}$, can both be viewed as pseudo prediction vectors that are ``better'' than $\Mvh y$, as they are both perfectly correct up to the $(k-1)$-th label. The two vectors only differ on the $k$-th prediction, which is correct for $\Mvh y^{(k)}_{\text{real}}$ and incorrect for $\Mvh y^{(k)}_{\text{pred}}$. The difference allows the term $\delta^{(k)}$ in (\ref{eq:label_weight}) to quantify the price that the algorithm needs to pay if the $k$-th prediction is wrong.
}
Then, the price $\delta^{(k)}$ can be viewed as an indicator of importance for predicting the $k$-th label correctly.
Our label-weighting scheme follows such intuition by simply setting the weight of $k$-th label as $\delta^{(k)}$.
The label-weighting scheme with (\ref{eq:label_weight}) is not only intuitive, but also
enjoys nice theoretical guarantee under a mild condition of $c(\cdot,\cdot)$,
as shown in the following lemma.
\begin{lemma}\label{lem:weighted_hamming_bound}
If $c(\Mv y, \Mv y^{(k)}_{\text{pred}}) - c(\Mv y, \Mv y^{(k)}_{\text{real}}) \geq 0$ holds for any $k$, $\Mv y$ and $\Mvh y$, 
then for any given $\Mv y$ and $\Mvh y$, we have 
\[
c(\Mv y, \Mvh y) = \sum_{k=1}^K \delta^{(k)} \bfunc{\Mv y[k] \neq \Mvh y[k]}
\]
\end{lemma}
  The condition of the lemma, which generally holds for reasonable cost functions, simply says that for any label, a correct prediction should enjoy a lower cost than an incorrect prediction.
  The proof of the lemma
  can be found in \ref{proof:weighted_hamming_bound}.
Lemma~\ref{lem:weighted_hamming_bound} transforms $c(\Mv y, \Mvh y)$ into the corresponding weighted Hamming loss,
and thus enables the optimization over general cost functions. 
Note that condition implies that correcting a wrongly-predicted label leads to no higher cost,
and is considered mild as general cost functions for MLC satisfy the condition.


Next, we propose CS-DPP, which extends DPP based on our proposed label-weighting scheme.
Define $\Mv C$ as
\begin{equation}\label{eq:delta}
\Mv C = \mbox{diag}(\sqrt{\delta^{(1)}},...,\sqrt{\delta^{(K)}}) 
\end{equation}

With $\Mv C$, which carries the cost information, we establish a theorem similar to Theorem~\ref{thm:hamming_bound} to upper-bound $c(\Mv y, \Mvh y)$.

\begin{theorem}\label{thm:cost_decomp}
When making a prediction~$\Mvh y$ from $\Mv x$ by $\Mvh y = \mathrm{round}\left(\Mv P^\top \Mv r(\Mv x) + \Mv o\right)$ 
with any left orthogonal matrix $\Mv P$, 
if $c(\cdot, \cdot)$ satisfies the condition of Lemma \ref{lem:weighted_hamming_bound}, 
the prediction cost 
\[
c(\Mv y, \Mvh y) \leq \|\Mv r(\Mv x) - \Mv z_{\Mv C}\|^2_2 + \|(\Mv I - \Mv P^\top\Mv P)(\Mv y_{\Mv C}')\|^2_2
\]
where $\Mv z_{\Mv C} = \Mv P (\Mv y_{\Mv C}')$ and $\Mv y_{\Mv C}' = \Mv C \Mv y - \Mv o$ with respect to any fixed reference point $\Mv o$. 
\end{theorem}

Theorem \ref{thm:cost_decomp} generalizes Theorem~\ref{thm:hamming_bound} to 
upper-bound the general cost $c(\Mv y, \Mvh y)$ instead of the original Hamming loss $c_{\textsc{ham}}(\Mv y, \Mvh y)$.
With Theorem~\ref{thm:cost_decomp}, extending DPP to CS-DPP is a straightforward task by reusing the online updating algorithms of DPP 
with $\Mv y_t$ replaced by $\Mv C_t \Mv y_t$.
The full details of CS-DPP using PBT is given in Algorithm~\ref{alg:CSDPP}, and we can easily write down similar steps for CS-DPP using PBC. Note that we simplify $\Mv W_t^{\textsc{PBT}}$ to $\Mv W_t$ in Algorithm~\ref{alg:CSDPP} to make a cleaner presentation.

\section{Experiments} \label{sec:exp}
To empirically evaluate the performance, and also to study the effectiveness and necessity of design components of CS-DPP, 
we conduct three sets of experiments:
(1) {necessity justification of online LSDR}, 
(2) {experiments on basis drifting}, and 
(3) {experiments on cost-sensitivity}.
Furthermore, recall that the label weighting scheme of CS-DPP depends on the label order.
We therefore conduct an additional set of experiments to study how different label orders affect the performance of CS-DPP.
\journalminor{
  To assist the readers in understanding the experiments, we list the full names and acronyms of the algorithms to be compared along with their key differences in Table~\ref{tbl:alg}. The details of the algorithms will be illustrated as needed.
}

  \begin{table}
\journalminor{
  \caption{Algorithms being compared in the experiments \label{tbl:alg}}
  \begin{tabular}{cp{1.5cm}p{1.5cm}p{1.5cm}p{1.5cm}p{1cm}p{1cm}p{1cm}}
    acronym & full name & dimension reduction & encode & basis transform & decode & cost-sensitivity\\ \hline \hline
    O-BR & Online Binary Relevance & - & no & - & - & no\\ \hline
    O-CS & Online Compressed Sensing & yes & random (static) & - & compressed sensing & no\\ \hline
    O-RAND & Online Random Projection & yes & random (static) & - & pseudo inverse & no\\ \hline \hline
    DPP-PBC & Dynamic Principal Projection (DPP) with Principal Basis Correction & yes & online PCA (dynamic) & exact & PCA & no\\ \hline
    DPP-PBT & Dynamic Principal Projection (DPP) with Principal Basis Transform (PBT) & yes & online PCA (dynamic) & approximate & PCA & no\\ \hline
    CS-DPP & Cost-Sensitive DPP (with PBT) & yes & online PCA & approximate & PCA & yes
  \end{tabular}
  }
  \end{table}

\subsection{Experiments Setup}
We conduct our experiments on eleven real-world datasets\footnote{$\Mt{CAL500}$, $\Mt{emotions}$, $\Mt{scene}$, $\Mt{yeast}$, $\Mt{enron}$, $\Mt{Corel5k}$, $\Mt{mediamill}$, $\Mt{nuswide}$, $\Mt{medical}$, $\Mt{delicious}$ and $\Mt{eurlex-eurovec}$} 
downloaded from Mulan~\cite{mulan}. 
Statistics of datasets can be found in Table~\ref{tbl:data_stats}.
In particular, datasets $\textit{eurlex-eurovec}$ and $\textit{delicious}$ are used only in the experiment to justify the necessity of online LSDR,
and only 7500 \journalminor{sub-sampled} instances are used on these two datasets to reduce the computational burden of the competitors in the experiment. 
In addition, only 50000 \journalminor{sub-sampled} instances are used for \textit{nuswide} because a competitor in the cost-sensitivity experiment is rather computationally inefficient. 
\begin{table}[t]
\centering
\caption{Statistics of datasets}
\label{tbl:data_stats}
\resizebox{0.7\textwidth}{!}{
\begin{tabular}{ccccc} \hline 
&\# of features & \# of labels &  \# of instances & cardinality \\ \hline
CAL500 & 68 & 174  & 502 & 26.044 \\
Corel5k & 499 & 374  & 5000 & 3.522 \\
emotions & 72 & 6  & 593 & 1.869 \\
enron & 1001 & 53  & 1702 & 3.378\\
mediamill & 120 & 101  & 43907 & 4.376\\
medical & 1449 & 45 & 978 & 1.245 \\ 
scene & 294 & 6  & 2407 & 1.074\\
yeast & 103 & 14  & 2417 & 4.237 \\ 
nuswide & 128 & 81  & 50000\textsuperscript{*} & 1.869 \\ \hline
delicious & 500 & 983  & 7500\textsuperscript{*} & 19.020 \\ 
eurlex-eurovec & 5000 & 3993 & 7500\textsuperscript{*} & 5.310 \\ \hline
\end{tabular}
}
\end{table}

\journalminor{
Data streams are generated by permuting datasets into different random orders. We perform sub-sampling on
\textit{eurlex-eurovec}, \textit{delicious} and \textit{nuswide} after computing the permutation so that each stream contains a diferent set of original instances for the three datasets.
}    

All LSDR algorithms, except for competitors run on \textit{delicious} and \textit{eurlex-eurovec}, are coupled with online ridge regression
and three different code space dimensions, $M = 10\%$, $25\%$, and $50\%$ of $K$, are considered.
For DPP we fix $\lambda = 1$ and follow~\cite{online_pca_sgd} to use the time-decreasing learning rate $\eta = \frac{2}{\sqrt{t}}\frac{M}{K}$,
and parameters of other algorithms will be elaborated along with their details in the corresponding section. 
For the two larger datasets \textit{delicious} and \textit{eurlex-eurovec}, we implement both DPP and O-BR using gradient descent instead of online ridge regression for calculating $\Mv W_t$,
where O-BR is the competitor that will be elaborated in Section~\ref{sec:exp:lsdr}.
In particular, 
for PBC of DPP we replace the \journallast{update of the online} ridge regressor (\ref{eq:W_PBC})
with online gradient descent,
while for PBT we replace (\ref{eq:W update PBT}), 
the update after basis transform, \journallast{with a gradient descent} update as well.
Note that even with online ridge regression replaced with gradient descent, \journallast{the ability} of DPP with PBT or PBC to handle the basis drifting problem remains unchanged. 
We use the time decreasing step-size $\frac{1}{\sqrt{t}}$ for gradient descent on \textit{delicious}, and  $\frac{0.001}{\sqrt{t}}$ on \textit{eurlex-eurovec}. 

We consider four different cost functions, Hamming loss, Normalized rank loss, F1 loss and Accuracy loss.
$$c_{\textsc{ham}}(\Mv y, \Mvh y) = \frac{1}{K} \left(\sum\limits_{k=1}^K \bfunc{\Mv y[k]\!\neq\!\Mvh y[k]}\right)$$
$$c_{\textsc{nr}}(\Mv y, \Mvh y) = \mathop{\mbox{average}}\limits_{\Mv y[i] > \Mv y[j]}\Bigl(\bfunc{\Mvh y[i]\!<\!\Mvh y[j]} + \frac{1}{2} \bfunc{\Mvh y[i]\!=\!\Mvh y[j]}\Bigr)$$
$$c_{\textsc{f1}}(\Mv y, \Mvh y) = 1 - 2 \left(\sum\limits_{k=1}^K \bfunc{\Mv y[k]\!=\!+1 \mbox{ and } \Mvh y[k] \!=\!+1}\right)/\left(\sum\limits_{k=1}^K (\bfunc{\Mv y[k]\!=\!+1} + \bfunc{\Mvh y[k]\!=\!+1})\right)$$
$$c_{\textsc{acc}}(\Mv y, \Mvh y)\!=\!1\!-\left(\sum\limits_{k=1}^K \bfunc{\Mv y[k]\!=\!+1 \mbox{ and } \Mvh y[k] \!=\!+1}\right) / \left(\sum\limits_{k=1}^K \bfunc{\Mv y[k]\!=\!+1 \mbox{ or } \Mvh y[k] \!=\!+1}\right)$$ 
The performances of different algorithms are compared using the average cumulative cost $\frac{1}{t}\sum_{i=1}^t c(\Mv y_i, \Mvh y_i)$ at each iteration $t$. 
We remark that \emph{lower} average cumulative cost imply better performance. 
We report the average results of each experiment after 15 repetitions.


\begin{table}[t]
\centering
\resizebox{0.95\textwidth}{!}{%
\begin{tabular}{l | c c c | c c c}
\hline
$\Mt{Dataset}$ & \multicolumn{3}{c|}{delicious} & \multicolumn{3}{c}{eurlex-eurovec} \\ 
\hline
$\Mt{Algorithms}$ & PBT & PBC & O-BR & PBT & PBC & O-BR  \\
\hline
$c_{\textsc{ham}}$ & ${0.1136}$ & $0.1153$ & $0.1245$ & ${0.4917}$ & $0.5011$ & $0.4993$  \\
$c_{\textsc{NR}}$ & ${0.5636}$ & $0.5641$ & $0.5756$ & ${0.7435}$ & $0.7467$ & ${0.7433}$  \\
$c_{\textsc{F1}}$ & $0.9143$ & $0.9138$ & ${0.9076}$ & $0.9972$ & $0.9928$ & ${0.9921}$  \\
$c_\textsc{{Acc}}$ & $0.9512$ & $0.9517$ & ${0.9494}$ & $0.9980$ & $0.9964$ & ${0.9958}$  \\
\hline
$\Mt{Avg.\:time\:(sec)}$ & $\mathbf{21.49}$ & $140.77$ & $105.18$ & $\mathbf{60.81}$ & 10522.25 & $4841.35$ \\ 
\end{tabular}
}
\caption{DPP vs. O-BR on large datasets}
\label{tbl:large-data}
\end{table}

\begin{figure}[t]
    \includegraphics[width=\linewidth]{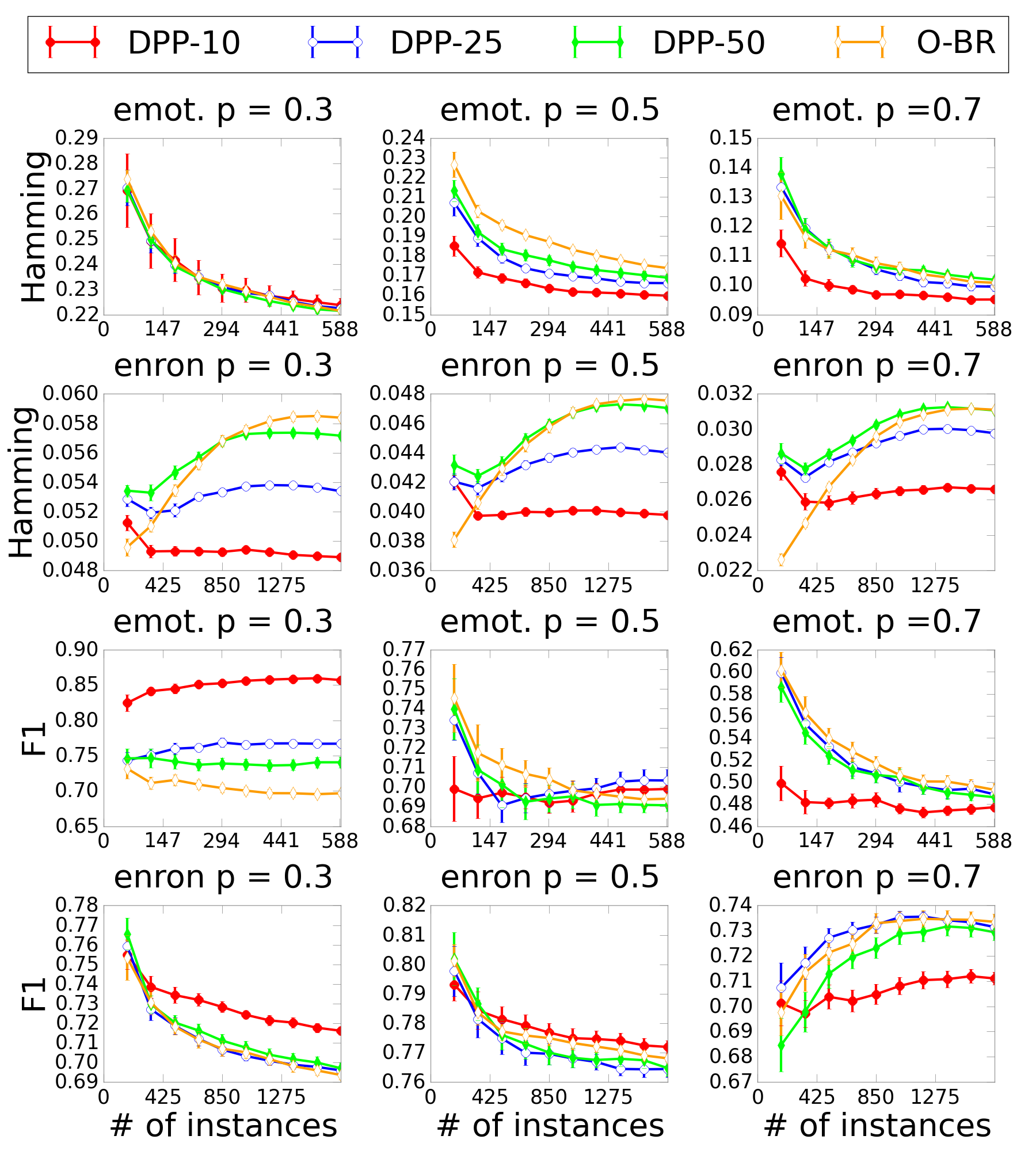}
    \caption{DPP vs. O-BR on noisy labels}
    \label{fig:5-2-1}
\end{figure}

\subsection{Necessity of Online LSDR}\label{sec:exp:lsdr} 
In this experiment, we aim to justify the necessity to address LSDR for OMLC problems.
We demonstrate that the ability of LSDR to preserve the key joint correlations between labels
can be helpful when facing
(1) data with noisy labels or (2) data with a large possible set of labels, 
which are often encountered in real-world OMLC problems. 
We compare DPP with online Binary Relevance (O-BR), 
which is a na\"{i}ve extension from binary relevance~\cite{intro} with online ridge regressor.
The only difference between DPP and O-BR is whether the algorithm incorporates LSDR.

We first compare DPP and O-BR on data with noisy labels.
We generate noisy data stream by randomly flipping each positive label $\Mv y[i] = 1$ 
to negative with probability $p = \{0.3, 0.5, 0.7\}$,
which simulates the real-world scenario in which human annotators fail to tag the existed labels.
We plot the results of O-BR and DPP with $M = 10\%$, $25\%$ and $50\%$ of $K$ on datasets \textit{emotions} and \textit{enron} 
with respect to Hamming loss and F1 loss in Figure~\ref{fig:5-2-1},
  \journallast{which contains error bars that represent the standard error of the average results.}
  The standard errors are naturally larger when $M$ is larger or when $t$ (number of iterations) is small, but in general for $M \ge 25\% \cdot K$ and for $t \ge 400$ the standard errors are small enough to justify the difference.
  The complete results are listed in \ref{appendix:exp:lsdr}.

The results from the first two rows of Figure~\ref{fig:5-2-1}
show that DPP with $M = 10\%$ of $K$ performs competitively and even better than O-BR as $p$ increases on dataset \textit{emotions}. 
The results from the last two rows of Figure~\ref{fig:5-2-1}
show that DPP always performs better on \textit{enron}. 
We can also observe from Figure~\ref{fig:5-2-1} that DPP with smaller $M$ tends to perform better as $p$ increases.
The above results clearly demonstrate that DPP better resists the effect of noisy labels with its incorporation of LSDR as the noise level ($p$) increases.
The observation that DPP with smaller $M$ tends to perform better demonstrates that 
DPP is more robust to noise by preserving the key of the key joint correlations between labels with LSDR.

Next, we demonstrate that LSDR is also helpful for handling data with a large label set. 
We compare O-BR with DPP that is coupled with either PBC or PBT on datasets 
\textit{delicious} 
and \textit{eurlex-eurovec}.\footnote{\textit{delicious}: $d$=$500$, $K$=$983$, \textit{eurlex-eurovec}: $d$=$5000$, $K$=$3993$.}
DPP uses $M = 10$ for \textit{delicious} and $M = 25$ for \textit{eurlex-eurovec}.
We summarize the results and average run-time in Table~\ref{tbl:large-data}. 
Table~\ref{tbl:large-data} indicates that DPP coupled with either PBT or PBC performs competitively with O-BR, 
while DPP with PBT enjoys significantly cheaper computational cost.
The results demonstrate that DPP enjoys more effective and efficient learning for data with a large label set than O-BR,
and also justifies the advantage of PBT over PBC in terms of efficiency when $K$ and $d$ are large while $M$ is relatively small,
as previously highlighted in Section~\ref{sec:DPP}.


\begin{figure}[t]
        \includegraphics[width=\linewidth]{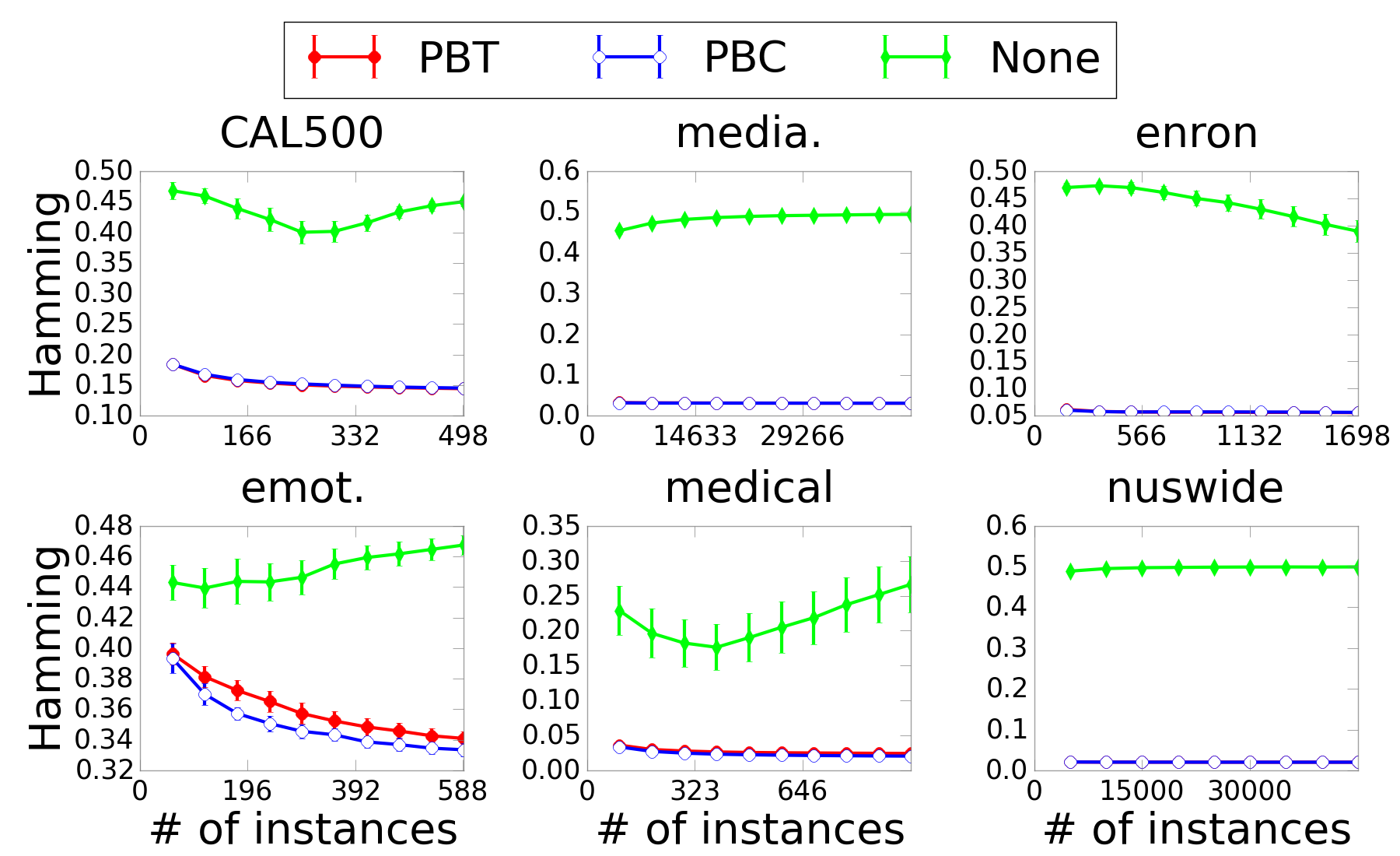}
    \caption{PBC vs. PBT vs. None, $M = 10\%$ of $K$}
    \label{fig:5-2-2}
\end{figure}

\subsection{Experiments on Basis Drifting}\label{sec:exp:basis drift}
To empirically justify the necessity of handling basis drifting, 
we compare variants of DPP that (a) incorporates PBC by (\ref{eq:W_PBC}), (b) incorporates PBT by (\ref{eq:W update PBT}), and (c) neglects basis drifting as (\ref{eq:W_naive}).
We plot the results for Hamming loss with $M=10\%$ of $K$ in
Figure~\ref{fig:5-2-2} on six datasets, 
and report the complete results in \ref{appendix:exp:drift}. 
The results on all datasets in Figure~\ref{fig:5-2-2} show that DPP with either PBC or PBT significantly improves the performance 
over its variant that neglects the basis drifting, which clearly demonstrates the necessity to handle the drifting of projection basis.

Further comparison of PBC and PBT based on Figure~\ref{fig:5-2-2} reveals that PBC in general performs slightly better than PBC,
reflecting its advantage of exact projection basis correction.
Nevertheless, as discussed in Section~\ref{sec:exp:lsdr},
PBT enjoys a nice computational speedup when $K$ and $d$ are large and $M$ is relatively small, 
making PBT more suitable to handle data with a large label set. 
\subsection{Experiments on Cost-Sensitivity}\label{sec:exp:cost}
To empirically justify the necessity of cost-sensitivity, we compare CS-DPP using PBT with DPP using PBT and other online LSDR algorithms.
To the best of our knowledge, no online LSDR algorithm has yet been proposed in the literature.
We therefore design two simple online LSDR algorithms, online compressed sensing (O-CS) and online random projection (O-RAND), 
to compare with CS-DPP. 
O-CS is a straightforward extension of CS~\cite{cs} with an online ridge regressor, and we follow~\cite{cs} to determine the parameter of O-CS. 
O-RAND encodes using random matrix $\Mv P_{R}$ and simply decodes with the corresponding pseudo inverse $\Mv P_{R}^{\dagger}$. 

We plot the results with respect to 
all evaluation criteria except for the Hamming loss with $M=10\%$ of $K$ in
Figure~\ref{fig:5-2-3} on three datasets, 
and report the complete results in \ref{appendix:exp:cost} 
Note that the results for CS-DPP here are obtained by using the original label order from the dataset. 

\begin{figure}[t]
    \includegraphics[width=\linewidth]{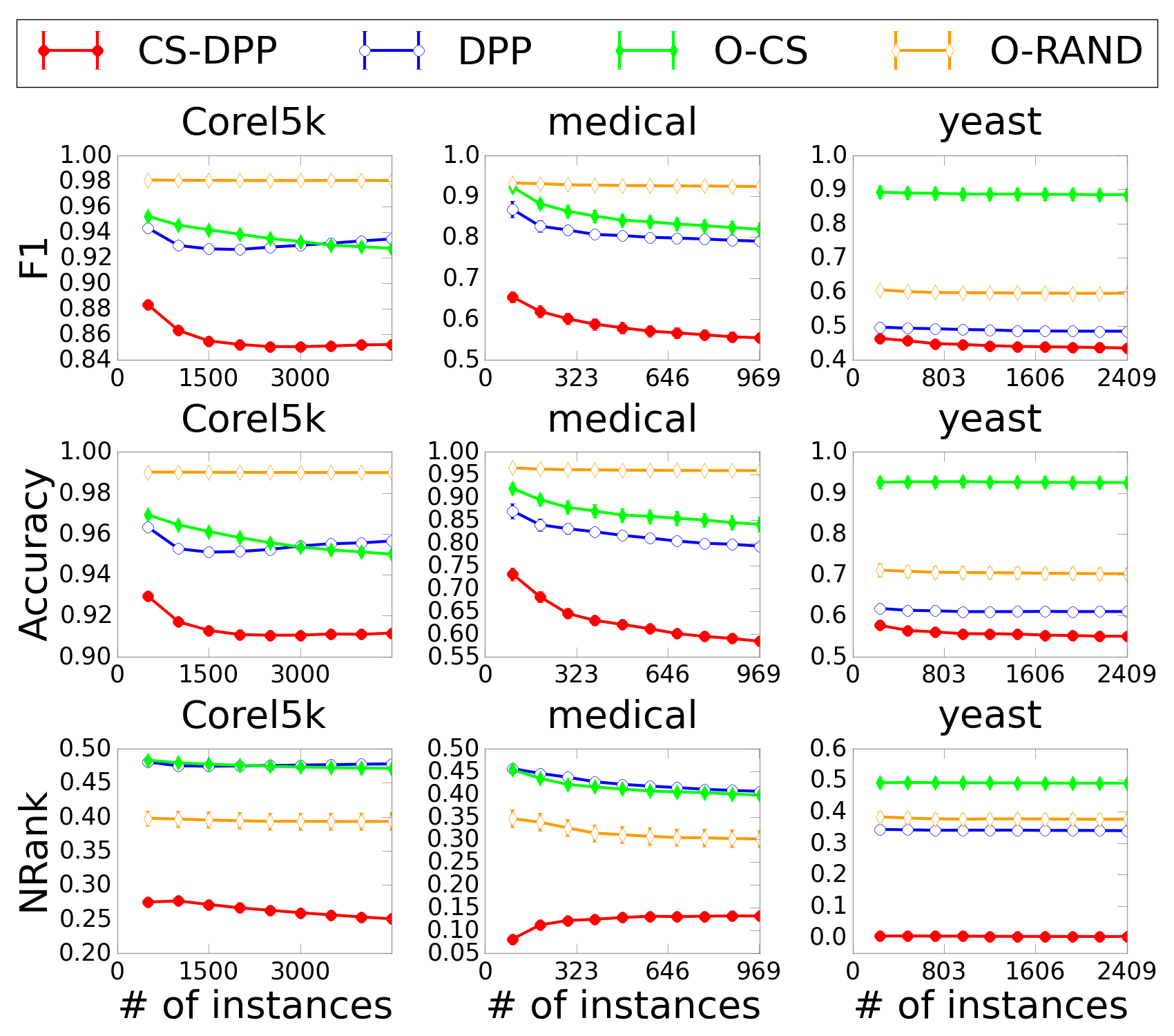}
    \caption{CS-DPP vs. Others, $M = 10\%$ of $K$}
    \label{fig:5-2-3}
\end{figure}

\subsubsection{CS-DPP versus DPP.} 
The results of Figure~\ref{fig:5-2-3} clearly indicate that CS-DPP performs significantly better than DPP 
on all evaluation criteria other than the Hamming loss, 
while CS-DPP reduces to DPP when $c_{\textsc{Ham}}(\cdot,\cdot)$ is used as the cost function.
These observations demonstrate that CS-DPP, by optimizing the given cost function instead of Hamming loss,
indeed achieves cost-sensitivity and is superior to its cost-insensitive counterpart, DPP.

\subsubsection{CS-DPP versus Other Online LSDR Algorithms.}
As shown in Figure~\ref{fig:5-2-3},
while DPP generally performs better than O-CS and O-RAND because of the advantage to 
preserve key label correlations rather than random ones, 
it can nevertheless be inferior on some datasets with respect to specific cost functions due to its cost-insensitivity.
For example, DPP loses to O-RAND on dataset \textit{Corel5k} with respect to the Normalized rank loss, 
as shown in the third row of Figure~\ref{fig:5-2-3}. 
CS-DPP conquers the weakness of DPP with its cost-sensitivity, and significantly outperforms O-CS and O-RAND 
on all three datasets with respect to all three evaluation criteria, as demonstrated in Figure~\ref{fig:5-2-3}.
The superiority of CS-DPP justifies the necessity to take cost-sensitivity into account.

\subsection{Experiment on Effect of Label Order for CS-DPP}
\begin{table}[t]
\centering
\begin{minipage}{0.8\textwidth}
\resizebox{0.99\textwidth}{!}{%
\begin{tabular}{l | c c c }
\hline
$\Mt{}$ & 10\% of K & 25\% of K & 50 \% of K   \\
\hline
$c_{\textsc{ham}}$ & $0.1458\pm 0.00019$ & $0.1489\pm 0.00012$ & $0.1503\pm 0.00008$ \\
$c_{\textsc{NR}}$  & $0.1247\pm 0.00224$ & $0.1321\pm 0.00210$ & $0.1371\pm 0.00222$ \\
$c_{\textsc{F1}}$  & $0.5914\pm 0.00108$ & $0.5956\pm 0.00110$ & $0.5949\pm 0.00101$ \\
$c_\textsc{{Acc}}$ & $0.7388\pm 0.00105$ & $0.7428\pm 0.00131$ & $0.7426\pm 0.00126$
\end{tabular}
}
\caption{Results of CS-DPP on \textit{CAL500} with 50 random label orders}
\vspace{0.5em}
\label{tbl:cost-order-CAL500}
\end{minipage}

\begin{minipage}{0.8\textwidth}
\resizebox{0.99\textwidth}{!}{%
\begin{tabular}{l | c c c }
\hline
$\Mt{}$ & 10\% of K & 25\% of K & 50 \% of K   \\
\hline
$c_{\textsc{ham}}$ & $0.2296\pm 0.00010$ & $0.2162\pm 0.00009$ & $0.2092\pm 0.00001$ \\
$c_{\textsc{NR}}$  & $0.0064\pm 0.00081$ & $0.0170\pm 0.00242$ & $0.0232\pm 0.00158$ \\
$c_{\textsc{F1}}$  & $0.4518\pm 0.00919$ & $0.3841\pm 0.00199$ & $0.3784\pm 0.00107$ \\
$c_\textsc{{Acc}}$ & $0.5448\pm 0.02252$ & $0.4971\pm 0.00379$ & $0.4901\pm 0.00124$ 
\end{tabular}
}
\caption{Results of CS-DPP on \textit{yeast} with 50 random label orders}
\vspace{0.5em}
\label{tbl:cost-order-yeast}
\end{minipage}

\begin{minipage}{0.8\textwidth}
\resizebox{0.99\textwidth}{!}{%
\begin{tabular}{l | c c c }
\hline
$\Mt{}$ & 10\% of K & 25\% of K & 50 \% of K   \\
\hline
$c_{\textsc{ham}}$ & $0.0562\pm 0.00020$ & $0.0600\pm 0.00011$ & $0.0632\pm 0.00009$ \\
$c_{\textsc{NR}}$  & $0.1432\pm 0.00333$ & $0.1364\pm 0.00244$ & $0.1305\pm 0.00216$ \\
$c_{\textsc{F1}}$  & $0.5421\pm 0.00334$ & $0.5392\pm 0.00291$ & $0.5428\pm 0.00293$ \\
$c_\textsc{{Acc}}$ & $0.6573\pm 0.00360$ & $0.6561\pm 0.00331$ & $0.6627\pm 0.00315$ 
\end{tabular}
}
\caption{Results of CS-DPP on \textit{enron} with 50 random label orders}
\label{tbl:cost-order-enron}
\end{minipage}
\end{table}

The goal of this experiment is to study how different label orders affect the performance of CS-DPP 
as our proposed label weighting scheme with (\ref{eq:label_weight}) is label-order-dependent. 
To evaluate the impact of label orders, 
we run CS-DPP with $50$ randomly generated label orders and $M = 10\%$, $25\%$ and $50\%$ of $K$ on each dataset. 
The permutation of each dataset is fixed to the original one given in Mulan~\cite{mulan},
which allows the variance of the performance to better indicate the effect of different orders.

We summarize the results of all four different cost functions with mean and standard deviation
on datasets \textit{CAL500}, \textit{enron} and \textit{yeast} in 
Table~\ref{tbl:cost-order-CAL500},~\ref{tbl:cost-order-yeast} and~\ref{tbl:cost-order-enron} respectively,
and report the complete results in~\ref{appendix:exp:order}.
Note that the results of Hamming loss \journallast{are unaffected} by the order of labels, and the reported deviation is due to the randomness from $\Mv P_t$.
From the results of Table~\ref{tbl:cost-order-CAL500},~\ref{tbl:cost-order-yeast} and~\ref{tbl:cost-order-enron},
we see that standard deviation is generally in a relatively small scale of $10^{-3}$, indicating that 
the performance of CS-DPP \journallast{is not that sensitive} to the order of labels.
Closer \journallast{inspection of Table~\ref{tbl:cost-order-yeast}} reveals that 
the standard deviation of $c_{\textsc{acc}}$ on \textit{yeast} with $M = 10\%$ of $K$ (which is only $2$ in this case) is somewhat larger, but for sufficiently large $M$ the label order does not seem to cause much variation.



\section{Conclusion}\label{sec:conclusion}

We proposed a novel cost-sensitive online LSDR algorithm called cost-sensitive dynamic principal projection (CS-DPP). 
We established the foundation of CS-DPP with an online LSDR framework derived from PLST,
and derived CS-DPP along with its theoretical guarantees on top of MSG.
We successfully conquered the challenge of basis drifting using our carefully designed PBC and PBT. 
CS-DPP further achieves cost-sensitivity with theoretical guarantees based on our carefully designed label-weighting scheme. 
The empirical results demonstrate that CS-DPP significantly outperforms other OMLC algorithms on all evaluation criteria, 
which validates the robustness and superiority of CS-DPP.
The necessity for CS-DPP to address LSDR, basis drifting and cost-sensitivity was also empirically justified.

For possible future works, 
an interesting direction is to design an online LSDR algorithm capable of capturing the key joint information between features and labels.
As discussed, the concept to capture such joint information has been investigated for batch MLC~\cite{cplst,faie,leml}, but it remains
to be challenging for online MLC.
Another direction is to apply OMLC algorithms as a fast approximate solver for large-scale batch data, and see how they compete with traditional batch algorithms.
\journalminor{
Another interesting direction, as mentioned in Section~\ref{sec:related}, is to design online learning algorithms that achieve cost-sensitivity for the more sophisticated micro- and macro-based criteria.
}



\bibliographystyle{spmpsci}      
\bibliography{ecml2018}   

\begin{thebibliography}{10}
\providecommand{\url}[1]{{#1}}
\providecommand{\urlprefix}{URL }
\expandafter\ifx\csname urlstyle\endcsname\relax
  \providecommand{\doi}[1]{DOI~\discretionary{}{}{}#1}\else
  \providecommand{\doi}{DOI~\discretionary{}{}{}\begingroup
  \urlstyle{rm}\Url}\fi

\bibitem{online_pca_sgd}
Arora, R., Cotter, A., Srebro, N.: Stochastic optimization of {PCA} with capped
  {MSG}.
\newblock In: NIPS 2013, pp. 1815--1823 (2013)

\bibitem{landmark}
Balasubramanian, K., Lebanon, G.: The landmark selection method for multiple
  output prediction.
\newblock In: ICML 2012 (2012)

\bibitem{ridge_analysis}
Bartlett, P.: Online convex optimization: ridge regression, adaptivity (2008).
\newblock
  \urlprefix\url{https://people.eecs.berkeley.edu/~bartlett/courses/281b-sp08/24.pdf}

\bibitem{emotions}
Bello, J.P., Chew, E., Turnbull, D.: Multilabel classification of music into
  emotions.
\newblock In: ICMIR 2008, pp. 325--330 (2008)

\bibitem{sleec}
Bhatia, K., Jain, H., Kar, P., Varma, M., Jain, P.: Sparse local embeddings for
  extreme multi-label classification.
\newblock In: NIPS 2015, pp. 730--738 (2015)

\bibitem{cssp}
Bi, W., Kwok, J.T.: Efficient multi-label classification with many labels.
\newblock In: ICML 2013, pp. 405--413 (2013)

\bibitem{cplst}
Chen, Y., Lin, H.: Feature-aware label space dimension reduction for
  multi-label classification.
\newblock In: NIPS 2012, pp. 1538--1546 (2012)

\bibitem{nus-wide}
Chua, T., Tang, J., Hong, R., Li, H., Luo, Z., Zheng, Y.: {NUS-WIDE:} a
  real-world web image database from national university of singapore.
\newblock In: CIVR 2009 (2009)

\bibitem{PA}
Crammer, K., Dekel, O., Keshet, J., S.{-}S., S., Singer, Y.: Online
  passive-aggressive algorithms.
\newblock Journal of Machine Learning Research \textbf{7}, 551--585 (2006)

\bibitem{pcc}
Dembczynski, K., Cheng, W., H{\"{u}}llermeier, E.: Bayes optimal multilabel
  classification via probabilistic classifier chains.
\newblock In: ICML 2010, pp. 279--286 (2010)

\bibitem{pcc_f1}
Dembczynski, K., Waegeman, W., Cheng, W., H{\"{u}}llermeier, E.: An exact
  algorithm for {F}-measure maximization.
\newblock In: NIPS 2011, pp. 1404--1412 (2011)

\bibitem{yeast}
Elisseeff, A., Weston, J.: A kernel method for multilabelled classification.
\newblock In: NIPS 2001 (2001)

\bibitem{cs}
Hsu, D., Kakade, S., Langford, J., Zhang, T.: Multi-label prediction via
  compressed sensing.
\newblock In: NIPS 2009, pp. 772--780 (2009)

\bibitem{bcs}
Kapoor, A., Viswanathan, R., Jain, P.: Multilabel classification using bayesian
  compressed sensing.
\newblock In: NIPS 2012, pp. 2654--2662 (2012)

\bibitem{cft}
Li, C., Lin, H.: Condensed filter tree for cost-sensitive multi-label
  classification.
\newblock In: ICML 2014, pp. 423--431 (2014)

\bibitem{online_pca_spca}
Li, C., Lin, H., Lu, C.: Rivalry of two families of algorithms for
  memory-restricted streaming pca.
\newblock In: AISTATS 2016 (2016)

\bibitem{faie}
Lin, Z., Ding, G., Hu, M., Wang, J.: Multi-label classification via
  feature-aware implicit label space encoding.
\newblock In: ICML 2014, pp. 325--333 (2014)

\bibitem{Liu17}
Liu, W., Tsang, I.W., M{\"u}ller, K.R.: An easy-to-hard learning paradigm for
  multiple classes and multiple labels.
\newblock Journal of Machine Learning Research  (2017)

\bibitem{cs-rakel}
Lo, H., Wang, J., Wang, H., Lin, S.: Cost-sensitive multi-label learning for
  audio tag annotation and retrieval.
\newblock {IEEE} Trans. Multimedia \textbf{13}(3), 518--529 (2011)

\bibitem{Mao13}
Mao, Q., Tsang, I.W.H., Gao, S.: Objective-guided image annotation.
\newblock IEEE Transactions on Image Processing  (2013)

\bibitem{online_pca_optimal_regret}
Nie, J., Kotlowski, W., Warmuth, M.K.: Online {PCA} with optimal regrets.
\newblock Journal of Machine Learning Research \textbf{17}, 194--200 (2016)

\bibitem{stream_mlc_mtt}
Osojnik, A., Panov, P., Džeroski, S.: Multi-label classification via
  multi-target regression on data streams.
\newblock Machine Learning  (2017)

\bibitem{stream_mlc}
Read, J., Bifet, A., Holmes, G., Pfahringer, B.: Streaming multi-label
  classification.
\newblock In: Proceedings of the Workshop on Applications of Pattern Analysis
  (WAPA) 2011, pp. 19--25 (2011)

\bibitem{cc}
Read, J., Pfahringer, B., Holmes, G., Frank, E.: Classifier chains for
  multi-label classification.
\newblock Machine Learning \textbf{85}(3), 333--359 (2011)

\bibitem{cca}
Sun, L., Ji, S., Ye, J.: Canonical correlation analysis for multilabel
  classification: {A} least-squares formulation, extensions, and analysis.
\newblock {IEEE} TPAMI \textbf{33}(1), 194--200 (2011)

\bibitem{plst}
Tai, F., Lin, H.: Multilabel classification with principal label space
  transformation.
\newblock Neural Computation \textbf{24}(9), 2508--2542 (2012)

\bibitem{eval}
Tang, L., Rajan, S., Narayanan, V.K.: Large scale multi-label classification
  via metalabeler.
\newblock In: WWW 2009, pp. 211--220 (2009)

\bibitem{intro}
Tsoumakas, G., Katakis, I., Vlahavas, I.P.: Mining multi-label data.
\newblock In: Data Mining and Knowledge Discovery Handbook, 2nd ed., pp.
  667--685 (2010)

\bibitem{rakel}
Tsoumakas, G., Vlahavas, I.P.: Random \emph{k} -labelsets: An ensemble method
  for multilabel classification.
\newblock In: ECML 2007, pp. 406--417 (2007)

\bibitem{mulan}
Tsoumakas, G., Xioufis, E.S., Vilcek, J., Vlahavas, I.P.: {MULAN:} {A} java
  library for multi-label learning.
\newblock Journal of Machine Learning Research \textbf{12}, 2411--2414 (2011)

\bibitem{prakel}
Wu, Y., Lin, H.: Progressive $k$-labelsets for cost-sensitive multi-label
  classification.
\newblock Machine Learning  (2016).
\newblock Accepted for Special Issue of ACML 2016

\bibitem{cdp_stream_mlc}
Xioufis, E.S., Spiliopoulou, M., Tsoumakas, G., Vlahavas, I.P.: Dealing with
  concept drift and class imbalance in multi-label stream classification.
\newblock In: IJCAI 2011, pp. 1583--1588 (2011)

\bibitem{leml}
Yu, H., Jain, P., Kar, P., Dhillon, I.S.: Large-scale multi-label learning with
  missing labels.
\newblock In: ICML 2014, pp. 593--601 (2014)

\bibitem{bayes_stream_mlc}
Zhang, X., Graepel, T., Herbrich, R.: Bayesian online learning for multi-label
  and multi-variate performance measures.
\newblock In: AISTATS 2010 (2010)

\end{thebibliography}
\setcounter{section}{0}
\setcounter{subsection}{0}
\renewcommand\thesection{Appendix \Alph{section}.}
\renewcommand\thesubsection{\thesection\arabic{subsection}}

\addtocounter{theorem}{-4}
\section{Proof of Lemmas and Theorems}\label{appendix:proof}
\subsection{Proof of Lemma~\ref{lem:P sampling}}
\label{proof:P sampling}

\begin{lemma}
Suppose $(\Mv Q_t, \sigma_t)$ is obtained after an updated of Capped MSG such that $\Mv U_t = \Mv Q_t \mbox{diag}(\sigma_t) \Mv Q_t^\top$.
If $\Gamma_t$ is a discrete probability distribution over events $\{\Mv Q_t^{-i}\}_{i=1}^{M+1}$ with probability of $\Mv Q_t^{-i}$ being $1-\sigma_t[i]$,
we have for any $\Mv y$ 
\begin{equation}\label{eq:P sampling}
\mathbb{E}_{\Mv P_t \sim \Gamma_t} [\Mv y^\top (\Mv I - \Mv P_t^\top \Mv P_t) \Mv y ] = \Mv y^\top (\Mv I - \Mv U_t) \Mv y
\end{equation}
\end{lemma}
\begin{proof}
We first formally show that $\Gamma_t$ is a well-defined probability distribution. 
By the definition of the projection operator of Capped MSG we have $0 \leq \sigma_t[i] \leq 1$ for each $\sigma_t[i]$ and
$\sum_{i=1}^{M+1} 1 - \sigma_t[i] = M+1 - \sum_{i=1}^{M+1}\sigma_t[i] = 1$ with $tr(\Mv U_t) = M$.
$\Gamma_t$ is therefore a well-defined probability distribution.

Then it suffices to show that $\mathbb{E}_{\Mv P_t\sim \Gamma_t}[\Mv P_t^\top \Mv P_t] = \Mv U_t$ 
as 
\[
\mathbb{E}_{\Mv P_t \sim \Gamma _t}[\Mv y^\top (\Mv I - \Mv P_t^\top \Mv P_t) \Mv y] = 
\|\Mv y\|_2^2 - \Mv y^\top \mathbb{E}_{\Mv P_t\sim \Gamma_t}[\Mv P_t^\top \Mv P_t] \Mv y
\]
To see that $\mathbb{E}_{\Mv P_t\sim \Gamma_t}[\Mv P_t^\top \Mv P_t] = \Mv U_t$, 
first notice that by orthogonality of rows of $\Mv Q_t$ we have $\Mv U_t = \sum_{j=1}^{M+1} \sigma_t(j) \Mv e_j\Mv e_j^\top$ 
where $\Mv e_j$ is the $j$-th row of $\Mv Q_t$. 
We then have
\[
\begin{aligned}
\mathbb{E}_{\Mv P_t \sim \Gamma_t}[\Mv P_t^\top \Mv P_t] & = \sum_{i=1}^{M+1} (1 - \sigma_t[i]) \sum_{j=1}^{M+1}\bfunc{i\neq j} \Mv e_j \Mv e_j^\top \\
        & = \sum_{j=1}^{M+1} (\Mv e_j \Mv e_j^\top \sum_{i=1}^{M+1}\bfunc{i \neq j}(1 - \sigma_t[i])) \\
        & = \sum_{j=1}^{M+1} (\sigma_t[j] \Mv e_j \Mv e_j^\top) && (a) \\
        & = \Mv U_t 
\end{aligned}
\]
where $(a)$ is by $\sum_{i=1}^{M+1} \sigma_t[i] = M$ 

\end{proof}

\subsection{Proof of Theorem~\ref{thm:DPP_drift}}

\journalminor{
\label{proof:DPP_drift}
\begin{theorem}
With the definitions of $\Mv H_t$ in \eqref{eq:ht}, $\Mv H^*$ in \eqref{eq:hstar}, $\frac{\mathcal{R}}{T}$ in \eqref{eq:regret} and $\Delta_t$ in \eqref{eq:subspace}, assume that $\|\Mv x_t\| \le 1$, $\|\Mv y_t\| \le 1$ and
    $\|\Mv H_t \Mv x_t - \Mv y_t\|_2^2 \le \epsilon$.
    \begin{enumerate}
    \item For any given $T$, the expected cumulative regret $\mathcal{R}$ is upper-bounded by
      \[
      (1+\epsilon) \sum_{t=1}^T \Delta_t + \frac{M}{2} \|\Mv H^*\|_F^2 + 2 \epsilon M d \log \left(1 + \frac{T}{d}\right).
      \]
    \item If $\lim_{T \rightarrow \infty} \Delta_T = 0$ and $\|\Mv H^*\|_F \le h^*$ across all iterations,
      $\lim_{T \rightarrow \infty} \frac{\mathcal{R}}{T} = 0$.
    \end{enumerate}
\end{theorem}
\begin{proof}

We start by separating the definition of $\mathcal{R}$ to two terms: one for how $\Mv P_t$ in MSG converges to $\Mv P^*$, and the other for how $\Mv W_t^{\text{PBC}}$ for $\Mv P_t$ in ridge regression differs to $\Mv W_\#$ for~$\Mv P^*$.
For simplicity, we will denote $\Mv W_t^{\text{PBC}}$ by $\Mv W_t$. Then,
\begin{equation*} 
\begin{aligned}
\mathcal{R} = & \sum_{t=1}^T\mathbb{E}_{\Mv P_t\sim \Gamma_t}[\ell^{(t)}(\Mv W_t, \Mv P_t) - \ell^{(t)}(\Mv W_\#, \Mv P^*)]\\
=&+ \sum_{t=1}^T
 \mathbb{E}_{\Mv P_t\sim \Gamma_t}[\|\Mv W_t^\top \Mv x_t - \Mv P_t \Mv y_t\|_2^2 + \|(\Mv I - \Mv P_t^\top \Mv P_t)\Mv y_t\|_2^2]\\
& - \sum_{t=1}^T \mathbb{E}_{\Mv P_t\sim \Gamma_t}[\|\Mv W_\#^\top \Mv x_t - \Mv P^* \Mv y_t\|_2^2 + \|(\Mv I - (\Mv P^*)^\top \Mv P^*)\Mv y_t\|_2^2]\\
=& +\underbrace{\sum_{t=1}^T
    \mathbb{E}_{\Mv P_t\sim \Gamma_t}[\|(\Mv I - \Mv P_t^\top \Mv P_t)\Mv y_t\|_2^2] - \|(\Mv I - (\Mv P^*)^\top \Mv P^*)\Mv y_t\|_2^2]}_{\mathcal{R}_{\text{MSG}}}\\
    &+ \underbrace{\sum_{t=1}^T \mathbb{E}_{\Mv P_t\sim \Gamma_t}[\|\Mv W_t^\top \Mv x_t - \Mv P_t \Mv y_t\|_2^2] - \|\Mv W_\#^\top \Mv x_t - \Mv P^* \Mv y_t\|_2^2}_{\mathcal{R}_{\text{ridge}}}
\end{aligned}
\end{equation*}

\noindent We can bound $\mathcal{R}_{\text{MSG}}$ first.
Let $\Mv U_t = \mathbb{E}_{\Mv P_t \sim \Gamma_t}[\Mv P_t^\top \Mv P_t]$ and $\Mv U^* = \left(\Mv P^*\right)^\top \Mv P^*$, by linearity of expectation,
\begin{align}
\mathcal{R}_{\text{MSG}} \nonumber 
& = \sum_{t=1}^T \Mv y_t^\top (\Mv U_t - \Mv U^*) \Mv y_t \nonumber \\
& \leq \sum_{t=1}^T \|\Mv U_t - \Mv U^*\|_2  \label{eq:MSG} \\
& \leq \sum_{t=1}^T \Delta_t \nonumber
\end{align}
where \eqref{eq:MSG} from the assumption of $\|\Mv y_t\|_2 \leq 1$ and the definition of the matrix $2$-norm.

Next, we bound $\mathcal{R}_{\text{ridge}}$.
With the definitions of $\Mv H_t$ in \eqref{eq:ht} and $\Mv H^*$ in \eqref{eq:hstar}, $\mathcal{R}_{\text{ridge}}$ can be further decomposed to
\[
\begin{aligned}
\mathcal{R}_{\text{ridge}} = & + \underbrace{\sum_{t=1}^T \left(\mathbb{E}_{\Mv P_t\sim \Gamma_t}[\|\Mv P_t(\Mv H_t^\top \Mv x_t - \Mv y_t)\|_2^2]  - \|\Mv P^*(\Mv H_t^\top \Mv x_t  - \Mv y_t)\|_2^2\right)}_{\mathcal{R}_1} \\ 
& + \underbrace{\sum_{t=1}^T \left(\|\Mv P^*(\Mv H_t^\top \Mv x_t - \Mv y_t)\|_2^2 -  \|\Mv P^*(\left(\Mv H^*\right)^\top \Mv x_t - \Mv y_t)\|_2^2\right)}_{\mathcal{R}_2}.
\end{aligned}
\]
Bounding $\mathcal{R}_1$ is very similar to bounding $\mathcal{R}_{\text{MSG}}$. In particular, 
\begin{align}
\mathcal{R}_{1} \nonumber 
& = \sum_{t=1}^T (\Mv H_t^\top \Mv x_t - \Mv y_t)^\top (\Mv U_t - \Mv U^*) (\Mv H_t^\top \Mv x_t - \Mv y_t) \nonumber \\
& \leq \sum_{t=1}^T \epsilon \|\Mv U_t - \Mv U^*\|_2  \label{eq:R1} \\
& \leq \sum_{t=1}^T \epsilon \Delta_t \nonumber
\end{align}
where \eqref{eq:R1} follows from the assumption of $\|\Mv H_t^\top \Mv x_t - \Mv y_t\|_2^2 \le \epsilon$.

The term $\mathcal{R}_2$ can be viewed as an online ridge regression process from $\Mv x$ to $\Mv P^* \Mv y$, because it can be easily proved that $\Mv H_t (\Mv P^*)^\top$ is the ridge regression solution after receiving $(\Mv x_1, \Mv P^* \Mv y_1)$, $(\Mv x_2, \Mv P^* \Mv y_2)$, $\ldots$, $(\Mv x_{t-1}, \Mv P^* \Mv y_{t-1})$. Also, as discussed in Section~\ref{sec:W_learning}, $\Mv W_\# = \Mv H^* (\Mv P^*)^\top$ is the optimal linear regression solution of $\{(\Mv x_t, \Mv P^* \Mv y_t)\}_{t=1}^T$. The assumption of $\|\Mv H_t \Mv x_t - \Mv y_t\|_2^2\le \epsilon$ implies that
\begin{align}
\|\Mv P^* \Mv H_t^\top \Mv x_t - \Mv P^* \Mv y_t\|_2^2 = (\Mv H_t^\top \Mv x_t - \Mv y_t)^\top \Mv U^* (\Mv H_t^\top \Mv x_t - \Mv y_t) \le \epsilon \nonumber
\end{align}
as well. Similarly, he assumption of $\|\Mv y_t\|_2 \le 1$ implies that $\|\Mv P^* y_t\|_2 \le 1$. Then, a standard ridge regression analysis (see, e.g. \cite{ridge_analysis}) by provng that $\Mv A_t = \lambda \Mv I + \sum_{i=1}^{t-1} \Mv x_i \Mv x_i^\top$ grows linearly with $t$ leads to
\begin{align}
\mathcal{R}_{2} \nonumber 
&=\sum_{t=1}^T \left(\|\Mv P^* (\Mv H_t^\top \Mv x_t - \Mv y_t)\|_2^2 -  \|\Mv P^* (\left(\Mv H^*\right)^\top \Mv x_t - \Mv y_t)\|_2^2\right) \nonumber\\
&\le \frac{1}{2}\|\Mv P^* \Mv H^*\|_F^2 + 2 \epsilon Md \log\left(1 + \frac{T}{d}\right) \nonumber\\
&\le \frac{M}{2}\|\Mv H^*\|_F^2 + 2 \epsilon Md \log\left(1 + \frac{T}{d}\right) \label{eq:Pstar}
\end{align}
where \eqref{eq:Pstar} is because $\|\Mv P^*\|_F^2 = tr(\Mv U^*) = M$.

Summing $\mathcal{R}_{\text{MSG}}$, $\mathcal{R}_1$ and $\mathcal{R}_2$ results in
\begin{align}
\mathcal{R} \le (1+\epsilon) \sum_{t=1}^T \Delta_t + \frac{M}{2}\|\Mv H^*\|_F^2 + 2 \epsilon M d \log(1 + \frac{T}{d}),
\end{align}
which proves the first part of the theorem. The second part easily follows because the convergence of a sequence implies the convergence of the mean.
\end{proof}
}

\subsection{Proof of Lemma~\ref{lem:weighted_hamming_bound}}
\label{proof:weighted_hamming_bound}
\begin{lemma}
If $c(\Mv y, \Mv y^{(k)}_{\text{pred}}) - c(\Mv y, \Mv y^{(k)}_{\text{real}}) \geq 0$ holds for any $k$, $\Mv y$ and $\Mvh y$, 
then for any given $\Mv y$ and $\Mvh y$ we have 
\begin{equation}\label{eq:weighted_hamming_bound}
c(\Mv y, \Mvh y) = \sum_{k=1}^K \delta^{(k)} \bfunc{\Mv y[k] \neq \Mvh y[k]}
\end{equation}
\end{lemma}
\begin{proof}
Recall the definition of $\Mv y^{(k)}_{\text{real}}$ and $\Mv y^{(k)_{\text{pred}}}$ to be
\[
\Mvh y_{\text{real}}^{(k)}[i] = 
\begin{cases}
    \Mv y[i] & \text{if} \, i \leq k \\
    \Mvh y[i] & \text{if} \, i > k \\
\end{cases}
\: \text{\textbf{and}} \:
\Mvh y_{\text{pred}}^{(k)}[i] = 
\begin{cases}
    \Mv y[i] & \text{if} \, i < k \\
    \Mvh y[i] & \text{if} \, i \geq k \\
\end{cases}
\]
and the definition of $\delta^{(k)}$ to be
\[
\delta^{(k)} = |c(\Mv y, \Mvh y_{\text{pred}}^{(k)}) - c(\Mv y, \Mvh y_{\text{real}}^{(k)})|
\]
Now define $k_i, i = 1,\hdots,L$ be the sequence of indices such that $\Mv y[k_i] \neq \Mvh y[k_i]$ for every $k_i$ and $k_i < k_{i+1}$.
If such $k_i$ does not exist than (\ref{eq:weighted_hamming_bound}) holds trivially by $c(\Mv y, \Mv y) = 0$.
Otherwise, by the condition of $c$ we have
\[
\begin{aligned}
&\sum_{k=1}^K \delta^{(k)} \bfunc{\Mv y[k] \neq \Mvh y[k]} && (a)\\
=\; & \sum_{k=1}^K (c(\Mv y, \Mvh y_{\text{pred}}^{(k)}) - c(\Mv y, \Mvh y_{\text{real}}^{(k)})) \bfunc{\Mv y[k] \neq \Mvh y[k]} \\
=\; & \sum_{i=1}^L c(\Mv y, \Mvh y_{\text{pred}}^{(k_i)}) - c(\Mv y, \Mvh y_{\text{real}}^{(k_i)}) \\
=\; & c(\Mv y, \Mvh y_{\text{pred}}^{(k_1)}) - c(\Mv y, \Mvh y_{\text{real}}^{(k_L)})  && (b) \\ 
=\; & c(\Mv y, \Mvh y) && (c) \\
\end{aligned}
\]
where $(a)$ uses the condition of $c(\cdot,\cdot)$ to remove the absolute value function;
$(b)$ is from two possibilities of $L$: if $L = 1$ then the equation trivially holds; if $L > 1$
we use the observation that $\Mvh y_{\text{real}}^{(k_i)} = \Mvh y_{\text{pred}}^{(k_{i+1})}$
where the observation is by realizing $\Mv y[j] = \Mvh y[j]$ for any $k_i < j < k_{i+1}$;
$(c)$ follows from the observation that $\Mvh y^{(k_1)}_{\text{pred}} = \Mvh y$ and $\Mvh y^{(k_L)}_{\text{real}} = \Mv y$
and $c(\Mv y, \Mv y) = 0$.
\end{proof}

\subsection{Proof of Theorem~\ref{thm:cost_decomp}}
\begin{theorem}
When making a prediction~$\Mvh y$ from $\Mv x$ by $\Mvh y = \mathrm{round}\left(\Mv P^\top \Mv r(\Mv x) + \Mv o\right)$ 
with any left orthogonal matrix $\Mv P$, 
if $c(\cdot, \cdot)$ satisfies the condition of Lemma \ref{lem:weighted_hamming_bound}, 
the prediction cost 
\[
c(\Mv y, \Mvh y) \leq \|\Mv r(\Mv x) - \Mv z_{\Mv C}\|^2_2 + \|(\Mv I - \Mv P^\top\Mv P)(\Mv y_{\Mv C}')\|^2_2
\]
where $\Mv z_{\Mv C} = \Mv P (\Mv y_{\Mv C}')$ and $\Mv y_{\Mv C}' = \Mv C \Mv y - \Mv o$ with respect to any fixed reference point $\Mv o$. 
\end{theorem}

Recall the definition of $\Mv  C$ in the main context is
\begin{equation}\label{eq:C}
\Mv C = \mbox{diag}(\sqrt{\delta^{(1)}},...,\sqrt{\delta^{(K)}}) 
\end{equation}
Next we show and prove the following lemma before we proceed to the complete proof.
\begin{lemma}\label{lem:1}
Given the ground truth $\Mv y$, if the binary-value prediction $\Mvh y \in \{+1,-1\}^K$ is made by 
$\mathrm{round}(\Mvt y)$ where $\Mvt y$ is the real-value prediction $\Mvt y \in \mathbb{R}^K$.
Then for any $\Mv y$, $\Mvh y$, $\Mvt y$, if $c$ satisfies the condition in Lemma~\ref{lem:weighted_hamming_bound}, we have
\begin{equation}
c(\Mv y, \Mvh y) \leq \|\Mv C\Mv y - \Mvt y\|^2
\end{equation}
\end{lemma}

\begin{proof}
From Lemma~\ref{lem:weighted_hamming_bound} we have $c(\Mv y, \Mvh y) = \sum_{k=1}^K \delta^{(k)}\bfunc{\Mv y[k] \neq \Mvh y[k]}$.
As $\|\Mv C\Mv y - \Mvt y\|_2^2 = \sum_{k=1}^K (\sqrt{\delta^{(K)}}\Mv y[k] - \Mvt y[k])^2$,
it suffices to show that for all $k$ we have 
\begin{equation}\label{eq:bit_bound}
\delta^{(k)}\bfunc{\Mv y[k] \neq \Mvh y[k]} \leq (\sqrt{\delta^{(k)}}\Mv y[k] - \Mvt y[k])^2
\end{equation}

When $\delta^{(k)} = 0$, (\ref{eq:bit_bound}) holds trivially. 
When $\delta^{(k)} > 0$, we have
\[
\begin{aligned}
& \delta^{(k)} \bfunc{\Mv y[k] \neq \Mvh y[k]}  \\
                                 = \; &  \delta^{(k)}(\bfunc{\Mvt y[k] \geq 0}\bfunc{\Mv y[k] = -1} 
                                            + \bfunc{\Mvt y[k] < 0}\bfunc{\Mv y[k] = +1}) \\
                                 = \;& \delta^{(k)}(\bfunc{\frac{\Mvt y[k]}{\sqrt{\delta^{(k)}}} \geq 0}\bfunc{\Mv y[k] = -1} 
                                  + \bfunc{\frac{\Mvt y[k]}{\sqrt{\delta^{(k)}}} < 0}\bfunc{\Mv y[k] = +1}) \\
                                \leq \;& \delta^{(k)}((\frac{\Mvt y[k]}{\sqrt{\delta^{(k)}}} - \Mv y[k])^2\bfunc{\Mv y[k] = -1}  
                                + (\frac{\Mvt y[k]}{\sqrt{\delta^{(k)}}} - \Mv y[k])^2\bfunc{\Mv y[k] = +1})  \\
                               = \; & \delta^{(k)}(\frac{\Mvt y[k]}{\sqrt{\delta^{(k)}}} - \Mv y[k]) ^ 2 \\ 
                               = \; & (\sqrt{\delta^{(k)}}\Mv y[k] - \Mvt y[k]) ^ 2 
\end{aligned}
\]
where the second equality uses the fact that $\delta^{(k)} > 0$. 
As $\delta^{(k)} \geq 0$ holds by its definition, (\ref{eq:bit_bound}) holds for every $k$.
Summing (\ref{eq:bit_bound}) with respect to all $k$ then completes the proof.
\end{proof}

With Lemma~\ref{lem:1} established, we now prove Theorem~\ref{thm:cost_decomp}.
\begin{proof}[Proof of Theorem~\ref{thm:cost_decomp}]
If the given $c$ satisfies the condition in Lemma (\ref{lem:weighted_hamming_bound}),
and let $\Mvt y = \Mv P^\top \Mv r(\Mv x) + \Mv o$ and $\Mvh y = \mathrm{round}(\Mvt y)$.
Then for any $(\Mv x, \Mv y)$ we have 
\[
\begin{aligned}
& c(\Mv y, \Mvh y) \\ 
    \leq \;& \|\Mv C\Mv y - \Mvt y\|^2_2 && (a)\\
    =    \;& \|((\Mvt y - \Mv o - \Mv P^\top\Mv P \Mv y_{\Mv C}') - (\Mv y_{\Mv C}' - \Mv P^\top\Mv P \Mv y_{\Mv C}') ) \|^2_2 \\
    =    \;& \|(\Mv P^\top(\Mv r(\Mv x) - \Mv z^\Mv C) - (\Mv I- \Mv P^\top\Mv P)\Mv y_{\Mv C}'  \|^2_2 \\
    =    \;& \|(\Mv P^\top(\Mv r(\Mv x) - \Mv z^\Mv C)\|^2_2 + \|(\Mv I - \Mv P^\top\Mv P)\Mv y_{\Mv C}')\|^2_2  && (b) \\
    =    \;& \|\Mv r(\Mv x) - \Mv z_{\Mv C}\|^2_2 + \|(\Mv I - \Mv P^\top\Mv P)\Mv y_{\Mv C}'\|^2_2 && (c) \\
\end{aligned}
\]
where we recall that $\Mvb y_{\Mv C}' = \Mv C \Mv y - \Mv o$ and $\Mv z_{\Mv C} = \Mv P(\Mv y_{\Mv c}')$.
$(a)$ is from Lemma~\ref{eq:bit_bound}, while $(b)$ and $(c)$ follow from the orthogonal rows of $\Mv P$.
\end{proof}

We note that the proof above closely follows the proof of Theorem 1 in~\cite{plst}, while the key difference comes from Lemma~\ref{lem:1} to handle the weighted Hamming loss.

\section{Complete Results of Experiments}\label{appendix:exp}
Here we report the complete results of each experiment. 

\subsection{Necessity of Online LSDR}
\label{appendix:exp:lsdr}
We report the complete results of comparison between O-BR and DPP with $M = 10\%$, $25\%$ and $50\%$ of $K$ 
from Table~\ref{tbl:noise_comp_p30} to Table~\ref{tbl:noise_comp_p70} with respect to all four evaluation criteria, 
where the best values (the lowest) are marked in bold.

The results show that DPP outperforms O-BR as the value of $p$ increases with respect to Hamming loss, F1 loss and Accuracy loss,
demonstrating the robustness of DPP.
On the otter hand, the results related to Normalized rank loss from Table~\ref{tbl:noise_comp_p30} to Table~\ref{tbl:noise_comp_p70} show that, 
while DPP cannot outperform O-BR regarding this specific criterion, DPP does start to perform competitively as the value of $p$ increases. 
The observation again demonstrates that DPP indeed suffers less from noisy labels comparing to O-BR due to the incorporation with LSDR. 
\begin{table*}[h]
    \centering
    \resizebox{0.75\textwidth}{!}{
\begin{tabular}{cccccc }\hline
$\mathsf{Dataset}$ & $\mathsf{Alg.}$ & $\mathsf{Hamm. \, loss}$ & $\mathsf{F1 \, loss}$ & $\mathsf{Acc. \, loss}$ & $\mathsf{Norm. \, rank \, loss}$ \\ \hline
\multirow{4}{*}{$\mathsf{CAL500}$} & O-BR & ${0.1130 \pm 0.0003}$ & $\mathbf{0.823 \pm 0.002}$ & $\mathbf{0.453 \pm 0.001}$ & $\mathbf{0.898 \pm 0.001}$ \\ 
 & DPP-50 & ${0.1143 \pm 0.0001}$ & $\mathbf{0.823 \pm 0.002}$ & ${0.455 \pm 0.001}$ & $\mathbf{0.897 \pm 0.001}$ \\ 
 & DPP-25 & ${0.1133 \pm 0.0002}$ & ${0.830 \pm 0.003}$ & ${0.454 \pm 0.001}$ & $\mathbf{0.900 \pm 0.001}$ \\ 
 & DPP-10 & $\mathbf{0.1113 \pm 0.0002}$ & ${0.837 \pm 0.002}$ & ${0.458 \pm 0.001}$ & ${0.911 \pm 0.002}$ \\ \hline
\multirow{4}{*}{$\mathsf{Corel5k}$} & O-BR & $\mathbf{0.0070 \pm 0.0000}$ & ${0.949 \pm 0.001}$ & $\mathbf{0.496 \pm 0.000}$ & $\mathbf{0.957 \pm 0.001}$ \\ 
 & DPP-50 & ${0.0072 \pm 0.0000}$ & $\mathbf{0.945 \pm 0.001}$ & $\mathbf{0.496 \pm 0.001}$ & $\mathbf{0.957 \pm 0.001}$ \\ 
 & DPP-25 & ${0.0072 \pm 0.0000}$ & ${0.949 \pm 0.001}$ & $\mathbf{0.497 \pm 0.000}$ & $\mathbf{0.958 \pm 0.000}$ \\ 
 & DPP-10 & ${0.0071 \pm 0.0000}$ & ${0.949 \pm 0.001}$ & ${0.498 \pm 0.000}$ & ${0.960 \pm 0.001}$ \\ \hline
\multirow{4}{*}{$\mathsf{emotions}$} & O-BR & $\mathbf{0.2213 \pm 0.0011}$ & $\mathbf{0.697 \pm 0.005}$ & $\mathbf{0.480 \pm 0.004}$ & $\mathbf{0.719 \pm 0.005}$ \\ 
 & DPP-50 & $\mathbf{0.2214 \pm 0.0013}$ & ${0.740 \pm 0.008}$ & ${0.504 \pm 0.005}$ & ${0.764 \pm 0.004}$ \\ 
 & DPP-25 & $\mathbf{0.2226 \pm 0.0013}$ & ${0.767 \pm 0.006}$ & ${0.527 \pm 0.003}$ & ${0.783 \pm 0.002}$ \\ 
 & DPP-10 & $\mathbf{0.2238 \pm 0.0026}$ & ${0.857 \pm 0.003}$ & ${0.570 \pm 0.002}$ & ${0.858 \pm 0.004}$ \\ \hline
\multirow{4}{*}{$\mathsf{enron}$} & O-BR & ${0.0584 \pm 0.0002}$ & $\mathbf{0.694 \pm 0.002}$ & $\mathbf{0.386 \pm 0.001}$ & $\mathbf{0.766 \pm 0.002}$ \\ 
 & DPP-50 & ${0.0572 \pm 0.0002}$ & $\mathbf{0.697 \pm 0.003}$ & $\mathbf{0.388 \pm 0.001}$ & $\mathbf{0.770 \pm 0.002}$ \\ 
 & DPP-25 & ${0.0534 \pm 0.0002}$ & $\mathbf{0.696 \pm 0.002}$ & ${0.397 \pm 0.001}$ & $\mathbf{0.767 \pm 0.002}$ \\ 
 & DPP-10 & $\mathbf{0.0489 \pm 0.0001}$ & ${0.716 \pm 0.002}$ & ${0.414 \pm 0.001}$ & ${0.784 \pm 0.002}$ \\ \hline
\multirow{4}{*}{$\mathsf{mediamill}$} & O-BR & $\mathbf{0.0271 \pm 0.0000}$ & $\mathbf{0.640 \pm 0.001}$ & $\mathbf{0.403 \pm 0.000}$ & $\mathbf{0.721 \pm 0.000}$ \\ 
 & DPP-50 & ${0.0272 \pm 0.0000}$ & $\mathbf{0.640 \pm 0.001}$ & $\mathbf{0.402 \pm 0.000}$ & $\mathbf{0.721 \pm 0.000}$ \\ 
 & DPP-25 & ${0.0272 \pm 0.0000}$ & $\mathbf{0.639 \pm 0.001}$ & $\mathbf{0.403 \pm 0.000}$ & $\mathbf{0.721 \pm 0.001}$ \\ 
 & DPP-10 & ${0.0272 \pm 0.0000}$ & $\mathbf{0.639 \pm 0.001}$ & $\mathbf{0.402 \pm 0.000}$ & $\mathbf{0.720 \pm 0.001}$ \\ \hline
\multirow{4}{*}{$\mathsf{medical}$} & O-BR & $\mathbf{0.0168 \pm 0.0001}$ & $\mathbf{0.550 \pm 0.004}$ & $\mathbf{0.448 \pm 0.004}$ & $\mathbf{0.563 \pm 0.005}$ \\ 
 & DPP-50 & ${0.0177 \pm 0.0001}$ & $\mathbf{0.544 \pm 0.006}$ & $\mathbf{0.446 \pm 0.002}$ & $\mathbf{0.556 \pm 0.003}$ \\ 
 & DPP-25 & ${0.0183 \pm 0.0001}$ & ${0.577 \pm 0.004}$ & ${0.469 \pm 0.005}$ & ${0.589 \pm 0.005}$ \\ 
 & DPP-10 & ${0.0190 \pm 0.0001}$ & ${0.645 \pm 0.006}$ & ${0.538 \pm 0.003}$ & ${0.651 \pm 0.004}$ \\ \hline
\multirow{4}{*}{$\mathsf{nuswide}$} & O-BR & $\mathbf{0.0151 \pm 0.0000}$ & $\mathbf{0.627 \pm 0.001}$ & $\mathbf{0.668 \pm 0.000}$ & ${0.632 \pm 0.000}$ \\ 
 & DPP-50 & ${0.0151 \pm 0.0000}$ & $\mathbf{0.627 \pm 0.000}$ & $\mathbf{0.667 \pm 0.000}$ & ${0.633 \pm 0.000}$ \\ 
 & DPP-25 & ${0.0151 \pm 0.0000}$ & $\mathbf{0.627 \pm 0.000}$ & $\mathbf{0.667 \pm 0.000}$ & $\mathbf{0.632 \pm 0.000}$ \\ 
 & DPP-10 & ${0.0151 \pm 0.0000}$ & $\mathbf{0.626 \pm 0.000}$ & $\mathbf{0.668 \pm 0.000}$ & ${0.632 \pm 0.000}$ \\ \hline
\multirow{4}{*}{$\mathsf{scene}$} & O-BR & $\mathbf{0.1197 \pm 0.0005}$ & $\mathbf{0.626 \pm 0.001}$ & $\mathbf{0.560 \pm 0.002}$ & $\mathbf{0.628 \pm 0.003}$ \\ 
 & DPP-50 & ${0.1282 \pm 0.0008}$ & ${0.695 \pm 0.003}$ & ${0.622 \pm 0.002}$ & ${0.698 \pm 0.003}$ \\ 
 & DPP-25 & ${0.1273 \pm 0.0005}$ & ${0.706 \pm 0.003}$ & ${0.632 \pm 0.002}$ & ${0.710 \pm 0.004}$ \\ 
 & DPP-10 & ${0.1258 \pm 0.0004}$ & ${0.717 \pm 0.003}$ & ${0.643 \pm 0.001}$ & ${0.715 \pm 0.002}$ \\ \hline
\multirow{4}{*}{$\mathsf{yeast}$} & O-BR & $\mathbf{0.2034 \pm 0.0004}$ & $\mathbf{0.669 \pm 0.002}$ & $\mathbf{0.406 \pm 0.001}$ & $\mathbf{0.755 \pm 0.002}$ \\ 
 & DPP-50 & $\mathbf{0.2032 \pm 0.0004}$ & ${0.678 \pm 0.004}$ & ${0.413 \pm 0.002}$ & ${0.762 \pm 0.003}$ \\ 
 & DPP-25 & ${0.2045 \pm 0.0004}$ & ${0.711 \pm 0.004}$ & ${0.427 \pm 0.002}$ & ${0.783 \pm 0.003}$ \\ 
 & DPP-10 & ${0.2034 \pm 0.0005}$ & ${0.733 \pm 0.005}$ & ${0.443 \pm 0.002}$ & ${0.798 \pm 0.009}$ \\ \hline
\end{tabular}}

    \caption{DPP vs. O-BR on Noisy Data, $p$ = 0.3}
    \label{tbl:noise_comp_p30}
\end{table*}

\begin{table*}[h]
    \centering
    \resizebox{0.75\textwidth}{!}{
\begin{tabular}{cccccc }\hline
$\mathsf{Dataset}$ & $\mathsf{Alg.}$ & $\mathsf{Hamm. \, loss}$ & $\mathsf{F1 \, loss}$ & $\mathsf{Acc. \, loss}$ & $\mathsf{Norm. \, rank \, loss}$ \\ \hline
\multirow{4}{*}{$\mathsf{CAL500}$} & O-BR & $\mathbf{0.0815 \pm 0.0003}$ & $\mathbf{0.925 \pm 0.002}$ & ${0.483 \pm 0.001}$ & $\mathbf{0.961 \pm 0.001}$ \\ 
 & DPP-50 & ${0.0834 \pm 0.0001}$ & $\mathbf{0.925 \pm 0.002}$ & $\mathbf{0.480 \pm 0.001}$ & $\mathbf{0.962 \pm 0.001}$ \\ 
 & DPP-25 & ${0.0823 \pm 0.0002}$ & ${0.932 \pm 0.002}$ & ${0.483 \pm 0.000}$ & $\mathbf{0.961 \pm 0.001}$ \\ 
 & DPP-10 & $\mathbf{0.0816 \pm 0.0002}$ & ${0.947 \pm 0.002}$ & ${0.485 \pm 0.001}$ & ${0.970 \pm 0.001}$ \\ \hline
\multirow{4}{*}{$\mathsf{Corel5k}$} & O-BR & $\mathbf{0.0049 \pm 0.0000}$ & ${0.898 \pm 0.001}$ & $\mathbf{0.543 \pm 0.000}$ & $\mathbf{0.902 \pm 0.001}$ \\ 
 & DPP-50 & ${0.0051 \pm 0.0000}$ & ${0.899 \pm 0.001}$ & $\mathbf{0.544 \pm 0.001}$ & $\mathbf{0.900 \pm 0.001}$ \\ 
 & DPP-25 & ${0.0051 \pm 0.0000}$ & $\mathbf{0.897 \pm 0.001}$ & $\mathbf{0.544 \pm 0.000}$ & $\mathbf{0.900 \pm 0.001}$ \\ 
 & DPP-10 & ${0.0051 \pm 0.0000}$ & ${0.898 \pm 0.001}$ & $\mathbf{0.543 \pm 0.001}$ & $\mathbf{0.902 \pm 0.001}$ \\ \hline
\multirow{4}{*}{$\mathsf{emotions}$} & O-BR & ${0.1736 \pm 0.0014}$ & $\mathbf{0.694 \pm 0.004}$ & $\mathbf{0.633 \pm 0.003}$ & ${0.698 \pm 0.004}$ \\ 
 & DPP-50 & ${0.1689 \pm 0.0017}$ & $\mathbf{0.691 \pm 0.003}$ & ${0.640 \pm 0.003}$ & ${0.706 \pm 0.004}$ \\ 
 & DPP-25 & ${0.1660 \pm 0.0015}$ & $\mathbf{0.703 \pm 0.006}$ & ${0.646 \pm 0.002}$ & ${0.706 \pm 0.004}$ \\ 
 & DPP-10 & $\mathbf{0.1598 \pm 0.0014}$ & $\mathbf{0.699 \pm 0.004}$ & ${0.650 \pm 0.002}$ & $\mathbf{0.692 \pm 0.004}$ \\ \hline
\multirow{4}{*}{$\mathsf{enron}$} & O-BR & ${0.0475 \pm 0.0002}$ & ${0.768 \pm 0.001}$ & ${0.491 \pm 0.002}$ & $\mathbf{0.809 \pm 0.002}$ \\ 
 & DPP-50 & ${0.0470 \pm 0.0002}$ & $\mathbf{0.765 \pm 0.003}$ & $\mathbf{0.488 \pm 0.001}$ & $\mathbf{0.809 \pm 0.001}$ \\ 
 & DPP-25 & ${0.0440 \pm 0.0002}$ & $\mathbf{0.764 \pm 0.003}$ & ${0.491 \pm 0.001}$ & $\mathbf{0.806 \pm 0.002}$ \\ 
 & DPP-10 & $\mathbf{0.0398 \pm 0.0002}$ & ${0.772 \pm 0.002}$ & ${0.510 \pm 0.002}$ & $\mathbf{0.810 \pm 0.002}$ \\ \hline
\multirow{4}{*}{$\mathsf{mediamill}$} & O-BR & $\mathbf{0.0217 \pm 0.0000}$ & $\mathbf{0.831 \pm 0.001}$ & $\mathbf{0.548 \pm 0.000}$ & $\mathbf{0.840 \pm 0.001}$ \\ 
 & DPP-50 & $\mathbf{0.0217 \pm 0.0000}$ & $\mathbf{0.830 \pm 0.001}$ & ${0.550 \pm 0.000}$ & $\mathbf{0.839 \pm 0.001}$ \\ 
 & DPP-25 & $\mathbf{0.0217 \pm 0.0000}$ & $\mathbf{0.830 \pm 0.001}$ & ${0.550 \pm 0.001}$ & $\mathbf{0.840 \pm 0.001}$ \\ 
 & DPP-10 & $\mathbf{0.0217 \pm 0.0000}$ & $\mathbf{0.830 \pm 0.001}$ & ${0.549 \pm 0.000}$ & $\mathbf{0.840 \pm 0.001}$ \\ \hline
\multirow{4}{*}{$\mathsf{medical}$} & O-BR & $\mathbf{0.0153 \pm 0.0001}$ & ${0.570 \pm 0.002}$ & $\mathbf{0.655 \pm 0.005}$ & ${0.568 \pm 0.004}$ \\ 
 & DPP-50 & ${0.0163 \pm 0.0001}$ & ${0.563 \pm 0.005}$ & $\mathbf{0.661 \pm 0.003}$ & ${0.577 \pm 0.004}$ \\ 
 & DPP-25 & ${0.0160 \pm 0.0001}$ & ${0.569 \pm 0.004}$ & $\mathbf{0.664 \pm 0.005}$ & ${0.570 \pm 0.004}$ \\ 
 & DPP-10 & ${0.0157 \pm 0.0001}$ & $\mathbf{0.561 \pm 0.003}$ & ${0.690 \pm 0.003}$ & $\mathbf{0.565 \pm 0.003}$ \\ \hline
\multirow{4}{*}{$\mathsf{nuswide}$} & O-BR & $\mathbf{0.0109 \pm 0.0000}$ & ${0.537 \pm 0.000}$ & ${0.730 \pm 0.000}$ & $\mathbf{0.537 \pm 0.000}$ \\ 
 & DPP-50 & ${0.0110 \pm 0.0000}$ & ${0.537 \pm 0.000}$ & ${0.730 \pm 0.000}$ & $\mathbf{0.537 \pm 0.000}$ \\ 
 & DPP-25 & ${0.0110 \pm 0.0000}$ & $\mathbf{0.536 \pm 0.000}$ & ${0.730 \pm 0.000}$ & $\mathbf{0.536 \pm 0.000}$ \\ 
 & DPP-10 & ${0.0109 \pm 0.0000}$ & ${0.536 \pm 0.000}$ & $\mathbf{0.730 \pm 0.000}$ & $\mathbf{0.537 \pm 0.000}$ \\ \hline
\multirow{4}{*}{$\mathsf{scene}$} & O-BR & ${0.0965 \pm 0.0006}$ & ${0.533 \pm 0.003}$ & $\mathbf{0.718 \pm 0.002}$ & ${0.533 \pm 0.003}$ \\ 
 & DPP-50 & ${0.0926 \pm 0.0004}$ & ${0.525 \pm 0.002}$ & ${0.731 \pm 0.002}$ & ${0.524 \pm 0.002}$ \\ 
 & DPP-25 & ${0.0915 \pm 0.0004}$ & $\mathbf{0.519 \pm 0.003}$ & ${0.739 \pm 0.001}$ & ${0.522 \pm 0.003}$ \\ 
 & DPP-10 & $\mathbf{0.0902 \pm 0.0004}$ & ${0.524 \pm 0.003}$ & ${0.740 \pm 0.001}$ & $\mathbf{0.515 \pm 0.004}$ \\ \hline
\multirow{4}{*}{$\mathsf{yeast}$} & O-BR & ${0.1581 \pm 0.0005}$ & $\mathbf{0.853 \pm 0.002}$ & $\mathbf{0.518 \pm 0.001}$ & $\mathbf{0.875 \pm 0.001}$ \\ 
 & DPP-50 & ${0.1586 \pm 0.0005}$ & $\mathbf{0.850 \pm 0.002}$ & $\mathbf{0.520 \pm 0.001}$ & $\mathbf{0.873 \pm 0.002}$ \\ 
 & DPP-25 & ${0.1573 \pm 0.0004}$ & ${0.860 \pm 0.002}$ & ${0.524 \pm 0.001}$ & $\mathbf{0.878 \pm 0.002}$ \\ 
 & DPP-10 & $\mathbf{0.1543 \pm 0.0004}$ & ${0.876 \pm 0.004}$ & ${0.531 \pm 0.002}$ & ${0.890 \pm 0.002}$ \\ \hline
\end{tabular}}

    \caption{DPP vs. O-BR on Noisy Data, $p$ = 0.5}
    \label{tbl:noise_comp_p50}
\end{table*}

\begin{table*}[h]
    \centering
    \resizebox{0.75\textwidth}{!}{
\begin{tabular}{cccccc }\hline
$\mathsf{Dataset}$ & $\mathsf{Alg.}$ & $\mathsf{Hamm. \, loss}$ & $\mathsf{F1 \, loss}$ & $\mathsf{Acc. \, loss}$ & $\mathsf{Norm. \, rank \, loss}$ \\ \hline
\multirow{4}{*}{$\mathsf{CAL500}$} & O-BR & $\mathbf{0.0483 \pm 0.0003}$ & $\mathbf{0.985 \pm 0.001}$ & $\mathbf{0.495 \pm 0.000}$ & $\mathbf{0.990 \pm 0.001}$ \\ 
 & DPP-50 & ${0.0499 \pm 0.0002}$ & $\mathbf{0.983 \pm 0.001}$ & $\mathbf{0.495 \pm 0.000}$ & ${0.991 \pm 0.000}$ \\ 
 & DPP-25 & ${0.0502 \pm 0.0002}$ & $\mathbf{0.984 \pm 0.001}$ & ${0.495 \pm 0.000}$ & ${0.991 \pm 0.000}$ \\ 
 & DPP-10 & ${0.0490 \pm 0.0002}$ & ${0.987 \pm 0.000}$ & ${0.496 \pm 0.000}$ & ${0.992 \pm 0.001}$ \\ \hline
\multirow{4}{*}{$\mathsf{Corel5k}$} & O-BR & $\mathbf{0.0029 \pm 0.0000}$ & $\mathbf{0.716 \pm 0.002}$ & $\mathbf{0.647 \pm 0.001}$ & $\mathbf{0.716 \pm 0.002}$ \\ 
 & DPP-50 & ${0.0031 \pm 0.0000}$ & $\mathbf{0.713 \pm 0.001}$ & $\mathbf{0.646 \pm 0.001}$ & $\mathbf{0.714 \pm 0.002}$ \\ 
 & DPP-25 & ${0.0031 \pm 0.0000}$ & $\mathbf{0.714 \pm 0.001}$ & $\mathbf{0.647 \pm 0.001}$ & $\mathbf{0.715 \pm 0.002}$ \\ 
 & DPP-10 & ${0.0031 \pm 0.0000}$ & $\mathbf{0.712 \pm 0.002}$ & $\mathbf{0.646 \pm 0.001}$ & $\mathbf{0.714 \pm 0.002}$ \\ \hline
\multirow{4}{*}{$\mathsf{emotions}$} & O-BR & ${0.1007 \pm 0.0013}$ & ${0.493 \pm 0.006}$ & $\mathbf{0.759 \pm 0.002}$ & ${0.490 \pm 0.006}$ \\ 
 & DPP-50 & ${0.1017 \pm 0.0011}$ & ${0.486 \pm 0.005}$ & $\mathbf{0.758 \pm 0.002}$ & ${0.493 \pm 0.005}$ \\ 
 & DPP-25 & ${0.0993 \pm 0.0015}$ & ${0.489 \pm 0.006}$ & $\mathbf{0.757 \pm 0.002}$ & ${0.491 \pm 0.005}$ \\ 
 & DPP-10 & $\mathbf{0.0951 \pm 0.0013}$ & $\mathbf{0.477 \pm 0.004}$ & $\mathbf{0.763 \pm 0.002}$ & $\mathbf{0.474 \pm 0.004}$ \\ \hline
\multirow{4}{*}{$\mathsf{enron}$} & O-BR & ${0.0311 \pm 0.0002}$ & ${0.734 \pm 0.003}$ & $\mathbf{0.634 \pm 0.002}$ & ${0.753 \pm 0.003}$ \\ 
 & DPP-50 & ${0.0311 \pm 0.0002}$ & ${0.729 \pm 0.003}$ & $\mathbf{0.633 \pm 0.002}$ & ${0.745 \pm 0.002}$ \\ 
 & DPP-25 & ${0.0298 \pm 0.0002}$ & ${0.731 \pm 0.002}$ & $\mathbf{0.635 \pm 0.002}$ & ${0.742 \pm 0.003}$ \\ 
 & DPP-10 & $\mathbf{0.0266 \pm 0.0002}$ & $\mathbf{0.711 \pm 0.003}$ & ${0.644 \pm 0.001}$ & $\mathbf{0.726 \pm 0.003}$ \\ \hline
\multirow{4}{*}{$\mathsf{mediamill}$} & O-BR & $\mathbf{0.0130 \pm 0.0000}$ & $\mathbf{0.714 \pm 0.001}$ & $\mathbf{0.643 \pm 0.000}$ & ${0.715 \pm 0.000}$ \\ 
 & DPP-50 & ${0.0130 \pm 0.0000}$ & $\mathbf{0.715 \pm 0.000}$ & $\mathbf{0.643 \pm 0.000}$ & $\mathbf{0.714 \pm 0.000}$ \\ 
 & DPP-25 & ${0.0130 \pm 0.0000}$ & $\mathbf{0.714 \pm 0.000}$ & $\mathbf{0.643 \pm 0.000}$ & ${0.714 \pm 0.001}$ \\ 
 & DPP-10 & ${0.0130 \pm 0.0000}$ & $\mathbf{0.715 \pm 0.001}$ & $\mathbf{0.643 \pm 0.000}$ & ${0.715 \pm 0.001}$ \\ \hline
\multirow{4}{*}{$\mathsf{medical}$} & O-BR & $\mathbf{0.0099 \pm 0.0002}$ & ${0.398 \pm 0.007}$ & $\mathbf{0.814 \pm 0.003}$ & ${0.404 \pm 0.004}$ \\ 
 & DPP-50 & ${0.0106 \pm 0.0002}$ & ${0.401 \pm 0.005}$ & $\mathbf{0.812 \pm 0.003}$ & ${0.398 \pm 0.005}$ \\ 
 & DPP-25 & ${0.0105 \pm 0.0001}$ & ${0.391 \pm 0.004}$ & $\mathbf{0.815 \pm 0.002}$ & ${0.399 \pm 0.004}$ \\ 
 & DPP-10 & $\mathbf{0.0097 \pm 0.0001}$ & $\mathbf{0.377 \pm 0.004}$ & $\mathbf{0.819 \pm 0.003}$ & $\mathbf{0.377 \pm 0.005}$ \\ \hline
\multirow{4}{*}{$\mathsf{nuswide}$} & O-BR & $\mathbf{0.0066 \pm 0.0000}$ & $\mathbf{0.386 \pm 0.001}$ & ${0.808 \pm 0.000}$ & $\mathbf{0.386 \pm 0.001}$ \\ 
 & DPP-50 & ${0.0066 \pm 0.0000}$ & $\mathbf{0.386 \pm 0.000}$ & $\mathbf{0.807 \pm 0.000}$ & $\mathbf{0.385 \pm 0.000}$ \\ 
 & DPP-25 & ${0.0066 \pm 0.0000}$ & $\mathbf{0.386 \pm 0.000}$ & ${0.807 \pm 0.000}$ & $\mathbf{0.386 \pm 0.000}$ \\ 
 & DPP-10 & ${0.0066 \pm 0.0000}$ & $\mathbf{0.386 \pm 0.000}$ & ${0.807 \pm 0.000}$ & $\mathbf{0.385 \pm 0.001}$ \\ \hline
\multirow{4}{*}{$\mathsf{scene}$} & O-BR & ${0.0562 \pm 0.0004}$ & ${0.328 \pm 0.003}$ & $\mathbf{0.841 \pm 0.001}$ & ${0.328 \pm 0.002}$ \\ 
 & DPP-50 & $\mathbf{0.0544 \pm 0.0003}$ & ${0.323 \pm 0.002}$ & $\mathbf{0.841 \pm 0.001}$ & $\mathbf{0.321 \pm 0.002}$ \\ 
 & DPP-25 & $\mathbf{0.0542 \pm 0.0005}$ & $\mathbf{0.316 \pm 0.002}$ & $\mathbf{0.842 \pm 0.001}$ & $\mathbf{0.317 \pm 0.002}$ \\ 
 & DPP-10 & $\mathbf{0.0538 \pm 0.0005}$ & $\mathbf{0.313 \pm 0.002}$ & $\mathbf{0.842 \pm 0.001}$ & $\mathbf{0.318 \pm 0.002}$ \\ \hline
\multirow{4}{*}{$\mathsf{yeast}$} & O-BR & $\mathbf{0.0920 \pm 0.0004}$ & $\mathbf{0.746 \pm 0.002}$ & $\mathbf{0.625 \pm 0.001}$ & $\mathbf{0.747 \pm 0.002}$ \\ 
 & DPP-50 & $\mathbf{0.0918 \pm 0.0003}$ & $\mathbf{0.748 \pm 0.002}$ & $\mathbf{0.627 \pm 0.001}$ & $\mathbf{0.747 \pm 0.002}$ \\ 
 & DPP-25 & $\mathbf{0.0921 \pm 0.0004}$ & $\mathbf{0.747 \pm 0.002}$ & $\mathbf{0.627 \pm 0.001}$ & $\mathbf{0.746 \pm 0.002}$ \\ 
 & DPP-10 & $\mathbf{0.0915 \pm 0.0004}$ & $\mathbf{0.748 \pm 0.002}$ & $\mathbf{0.626 \pm 0.001}$ & $\mathbf{0.746 \pm 0.002}$ \\ \hline
\end{tabular}}

    \caption{DPP vs. O-BR on Noisy Data, $p$ = 0.7}
    \label{tbl:noise_comp_p70}
\end{table*}

\subsection{Experiments on Basis Drifting}
\label{appendix:exp:drift}
The complete results of comparison between DPP using (1) PBC, (2) PBT, and (3) nothing regarding Hamming loss can be found in Table~\ref{tbl:pb_comp_m10}, Table~\ref{tbl:pb_comp_m25} and Table~\ref{tbl:pb_comp_m50}, where the best values (the lowest) are marked in bold.
To further understand the behavior of basis drifting and the effectiveness of PBC and PBT for CS-DPP, 
we also compare CS-DPP coupled with PBC/PBT/none on F1 loss, Accuracy loss and Normalized rank loss, and summarize the results in the same tables.
From these results we can again draw the same conclusion as that in Section~\ref{sec:exp:basis drift}. That is, CS-DPP with either PBT or PBC greatly outperforms CS-DPP that neglects the basis drifting, and CS-DPP with PBT performs competitively with CS-DPP with PBC. 
\begin{table*}[h]
\centering
    \resizebox{0.75\textwidth}{!}{
\begin{tabular}{cccccc }\hline
$\mathsf{Dataset}$ & $\mathsf{Alg.}$ & $\mathsf{Hamm. \, loss}$ & $\mathsf{F1 \, loss}$ & $\mathsf{Acc. \, loss}$ & $\mathsf{Norm. \, rank \, loss}$ \\ \hline
\multirow{3}{*}{$\mathsf{CAL500}$} & CS-DPP-None & ${0.4464 \pm 0.0074}$ & ${0.733 \pm 0.001}$ & ${0.393 \pm 0.002}$ & ${0.843 \pm 0.001}$ \\ 
 & CS-DPP-PBT & $\mathbf{0.1443 \pm 0.0001}$ & $\mathbf{0.601 \pm 0.001}$ & $\mathbf{0.137 \pm 0.001}$ & $\mathbf{0.749 \pm 0.001}$ \\ 
 & CS-DPP-PBC & ${0.1454 \pm 0.0002}$ & $\mathbf{0.603 \pm 0.001}$ & ${0.144 \pm 0.002}$ & $\mathbf{0.748 \pm 0.001}$ \\ \hline
\multirow{3}{*}{$\mathsf{Corel5k}$} & CS-DPP-None & ${0.4814 \pm 0.0063}$ & ${0.957 \pm 0.000}$ & ${0.357 \pm 0.001}$ & ${0.980 \pm 0.000}$ \\ 
 & CS-DPP-PBT & $\mathbf{0.0099 \pm 0.0000}$ & ${0.853 \pm 0.001}$ & ${0.248 \pm 0.001}$ & ${0.912 \pm 0.001}$ \\ 
 & CS-DPP-PBC & ${0.0100 \pm 0.0000}$ & $\mathbf{0.850 \pm 0.001}$ & $\mathbf{0.237 \pm 0.001}$ & $\mathbf{0.910 \pm 0.000}$ \\ \hline
\multirow{3}{*}{$\mathsf{emotions}$} & CS-DPP-None & ${0.4787 \pm 0.0039}$ & ${0.618 \pm 0.004}$ & ${0.376 \pm 0.008}$ & ${0.696 \pm 0.005}$ \\ 
 & CS-DPP-PBT & ${0.3419 \pm 0.0033}$ & $\mathbf{0.445 \pm 0.003}$ & $\mathbf{0.159 \pm 0.021}$ & $\mathbf{0.563 \pm 0.007}$ \\ 
 & CS-DPP-PBC & $\mathbf{0.3301 \pm 0.0012}$ & $\mathbf{0.450 \pm 0.007}$ & $\mathbf{0.133 \pm 0.023}$ & $\mathbf{0.560 \pm 0.009}$ \\ \hline
\multirow{3}{*}{$\mathsf{enron}$} & CS-DPP-None & ${0.4030 \pm 0.0160}$ & ${0.802 \pm 0.002}$ & ${0.385 \pm 0.002}$ & ${0.875 \pm 0.001}$ \\ 
 & CS-DPP-PBT & $\mathbf{0.0560 \pm 0.0001}$ & ${0.534 \pm 0.002}$ & $\mathbf{0.124 \pm 0.003}$ & $\mathbf{0.642 \pm 0.002}$ \\ 
 & CS-DPP-PBC & ${0.0565 \pm 0.0001}$ & $\mathbf{0.528 \pm 0.002}$ & ${0.132 \pm 0.001}$ & $\mathbf{0.638 \pm 0.001}$ \\ \hline
\multirow{3}{*}{$\mathsf{mediamill}$} & CS-DPP-None & ${0.4936 \pm 0.0016}$ & ${0.692 \pm 0.016}$ & ${0.416 \pm 0.004}$ & ${0.728 \pm 0.001}$ \\ 
 & CS-DPP-PBT & ${0.0309 \pm 0.0000}$ & $\mathbf{0.460 \pm 0.000}$ & $\mathbf{0.066 \pm 0.002}$ & ${0.583 \pm 0.000}$ \\ 
 & CS-DPP-PBC & $\mathbf{0.0308 \pm 0.0000}$ & $\mathbf{0.460 \pm 0.000}$ & ${0.072 \pm 0.002}$ & $\mathbf{0.582 \pm 0.000}$ \\ \hline
\multirow{3}{*}{$\mathsf{medical}$} & CS-DPP-None & ${0.1923 \pm 0.0352}$ & ${0.896 \pm 0.002}$ & ${0.346 \pm 0.003}$ & ${0.932 \pm 0.003}$ \\ 
 & CS-DPP-PBT & ${0.0242 \pm 0.0001}$ & ${0.554 \pm 0.012}$ & ${0.132 \pm 0.005}$ & ${0.583 \pm 0.008}$ \\ 
 & CS-DPP-PBC & $\mathbf{0.0204 \pm 0.0002}$ & $\mathbf{0.508 \pm 0.006}$ & $\mathbf{0.096 \pm 0.003}$ & $\mathbf{0.549 \pm 0.007}$ \\ \hline
\multirow{3}{*}{$\mathsf{nuswide}$} & CS-DPP-None & ${0.4975 \pm 0.0006}$ & ${0.933 \pm 0.001}$ & ${0.520 \pm 0.001}$ & ${0.959 \pm 0.001}$ \\ 
 & CS-DPP-PBT & ${0.0201 \pm 0.0000}$ & $\mathbf{0.649 \pm 0.000}$ & $\mathbf{0.356 \pm 0.001}$ & $\mathbf{0.675 \pm 0.000}$ \\ 
 & CS-DPP-PBC & $\mathbf{0.0201 \pm 0.0000}$ & $\mathbf{0.648 \pm 0.000}$ & $\mathbf{0.358 \pm 0.001}$ & $\mathbf{0.675 \pm 0.000}$ \\ \hline
\multirow{3}{*}{$\mathsf{scene}$} & CS-DPP-None & ${0.4609 \pm 0.0080}$ & ${0.761 \pm 0.003}$ & ${0.362 \pm 0.007}$ & ${0.825 \pm 0.002}$ \\ 
 & CS-DPP-PBT & $\mathbf{0.1796 \pm 0.0001}$ & $\mathbf{0.723 \pm 0.002}$ & $\mathbf{0.264 \pm 0.012}$ & $\mathbf{0.798 \pm 0.003}$ \\ 
 & CS-DPP-PBC & $\mathbf{0.1797 \pm 0.0001}$ & $\mathbf{0.724 \pm 0.002}$ & $\mathbf{0.231 \pm 0.016}$ & $\mathbf{0.796 \pm 0.002}$ \\ \hline
\multirow{3}{*}{$\mathsf{yeast}$} & CS-DPP-None & ${0.4979 \pm 0.0015}$ & ${0.616 \pm 0.002}$ & ${0.422 \pm 0.003}$ & ${0.727 \pm 0.001}$ \\ 
 & CS-DPP-PBT & $\mathbf{0.2294 \pm 0.0010}$ & $\mathbf{0.435 \pm 0.004}$ & $\mathbf{0.003 \pm 0.000}$ & $\mathbf{0.549 \pm 0.003}$ \\ 
 & CS-DPP-PBC & $\mathbf{0.2307 \pm 0.0011}$ & $\mathbf{0.433 \pm 0.003}$ & $\mathbf{0.003 \pm 0.000}$ & $\mathbf{0.541 \pm 0.003}$ \\ \hline
\end{tabular}}

\caption{CS-DPP with PBC vs. PBT vs. None, $M$ = 10\% of $K$}
\label{tbl:pb_comp_m10}
\end{table*}

\begin{table*}[h]
\centering
    \resizebox{0.75\textwidth}{!}{
\begin{tabular}{cccccc }\hline
$\mathsf{Dataset}$ & $\mathsf{Alg.}$ & $\mathsf{Hamm. \, loss}$ & $\mathsf{F1 \, loss}$ & $\mathsf{Acc. \, loss}$ & $\mathsf{Norm. \, rank \, loss}$ \\ \hline
\multirow{3}{*}{$\mathsf{CAL500}$} & CS-DPP-None & ${0.4374 \pm 0.0100}$ & ${0.732 \pm 0.002}$ & ${0.392 \pm 0.002}$ & ${0.846 \pm 0.002}$ \\ 
 & CS-DPP-PBT & $\mathbf{0.1471 \pm 0.0002}$ & $\mathbf{0.604 \pm 0.001}$ & $\mathbf{0.151 \pm 0.002}$ & $\mathbf{0.750 \pm 0.001}$ \\ 
 & CS-DPP-PBC & ${0.1476 \pm 0.0001}$ & $\mathbf{0.602 \pm 0.001}$ & $\mathbf{0.150 \pm 0.002}$ & $\mathbf{0.751 \pm 0.001}$ \\ \hline
\multirow{3}{*}{$\mathsf{Corel5k}$} & CS-DPP-None & ${0.4997 \pm 0.0018}$ & ${0.965 \pm 0.000}$ & ${0.366 \pm 0.001}$ & ${0.983 \pm 0.000}$ \\ 
 & CS-DPP-PBT & $\mathbf{0.0100 \pm 0.0000}$ & $\mathbf{0.845 \pm 0.000}$ & ${0.223 \pm 0.001}$ & $\mathbf{0.905 \pm 0.000}$ \\ 
 & CS-DPP-PBC & ${0.0101 \pm 0.0000}$ & $\mathbf{0.844 \pm 0.000}$ & $\mathbf{0.220 \pm 0.001}$ & $\mathbf{0.904 \pm 0.000}$ \\ \hline
\multirow{3}{*}{$\mathsf{emotions}$} & CS-DPP-None & ${0.4988 \pm 0.0022}$ & ${0.631 \pm 0.004}$ & ${0.420 \pm 0.005}$ & ${0.722 \pm 0.003}$ \\ 
 & CS-DPP-PBT & $\mathbf{0.2768 \pm 0.0051}$ & $\mathbf{0.401 \pm 0.003}$ & $\mathbf{0.078 \pm 0.016}$ & $\mathbf{0.513 \pm 0.003}$ \\ 
 & CS-DPP-PBC & $\mathbf{0.2819 \pm 0.0036}$ & $\mathbf{0.398 \pm 0.004}$ & $\mathbf{0.046 \pm 0.015}$ & $\mathbf{0.509 \pm 0.003}$ \\ \hline
\multirow{3}{*}{$\mathsf{enron}$} & CS-DPP-None & ${0.4844 \pm 0.0050}$ & ${0.812 \pm 0.002}$ & ${0.386 \pm 0.002}$ & ${0.884 \pm 0.001}$ \\ 
 & CS-DPP-PBT & $\mathbf{0.0581 \pm 0.0002}$ & $\mathbf{0.517 \pm 0.001}$ & $\mathbf{0.136 \pm 0.002}$ & $\mathbf{0.633 \pm 0.001}$ \\ 
 & CS-DPP-PBC & ${0.0601 \pm 0.0002}$ & $\mathbf{0.519 \pm 0.001}$ & $\mathbf{0.135 \pm 0.001}$ & $\mathbf{0.633 \pm 0.001}$ \\ \hline
\multirow{3}{*}{$\mathsf{mediamill}$} & CS-DPP-None & ${0.4917 \pm 0.0015}$ & ${0.842 \pm 0.009}$ & ${0.429 \pm 0.001}$ & ${0.759 \pm 0.009}$ \\ 
 & CS-DPP-PBT & $\mathbf{0.0307 \pm 0.0000}$ & $\mathbf{0.458 \pm 0.000}$ & $\mathbf{0.070 \pm 0.000}$ & $\mathbf{0.581 \pm 0.000}$ \\ 
 & CS-DPP-PBC & $\mathbf{0.0307 \pm 0.0000}$ & $\mathbf{0.457 \pm 0.000}$ & $\mathbf{0.068 \pm 0.000}$ & $\mathbf{0.580 \pm 0.000}$ \\ \hline
\multirow{3}{*}{$\mathsf{medical}$} & CS-DPP-None & ${0.4493 \pm 0.0161}$ & ${0.902 \pm 0.002}$ & ${0.361 \pm 0.004}$ & ${0.931 \pm 0.004}$ \\ 
 & CS-DPP-PBT & ${0.0171 \pm 0.0002}$ & ${0.338 \pm 0.005}$ & ${0.043 \pm 0.003}$ & ${0.374 \pm 0.004}$ \\ 
 & CS-DPP-PBC & $\mathbf{0.0152 \pm 0.0001}$ & $\mathbf{0.316 \pm 0.004}$ & $\mathbf{0.036 \pm 0.002}$ & $\mathbf{0.360 \pm 0.004}$ \\ \hline
\multirow{3}{*}{$\mathsf{nuswide}$} & CS-DPP-None & ${0.4978 \pm 0.0007}$ & ${0.930 \pm 0.003}$ & ${0.523 \pm 0.000}$ & ${0.964 \pm 0.001}$ \\ 
 & CS-DPP-PBT & $\mathbf{0.0201 \pm 0.0000}$ & $\mathbf{0.648 \pm 0.000}$ & ${0.334 \pm 0.001}$ & $\mathbf{0.675 \pm 0.000}$ \\ 
 & CS-DPP-PBC & $\mathbf{0.0201 \pm 0.0000}$ & $\mathbf{0.648 \pm 0.000}$ & $\mathbf{0.329 \pm 0.001}$ & $\mathbf{0.675 \pm 0.000}$ \\ \hline
\multirow{3}{*}{$\mathsf{scene}$} & CS-DPP-None & ${0.5002 \pm 0.0012}$ & ${0.747 \pm 0.002}$ & ${0.373 \pm 0.004}$ & ${0.830 \pm 0.002}$ \\ 
 & CS-DPP-PBT & $\mathbf{0.1787 \pm 0.0014}$ & $\mathbf{0.632 \pm 0.003}$ & ${0.185 \pm 0.013}$ & $\mathbf{0.692 \pm 0.003}$ \\ 
 & CS-DPP-PBC & $\mathbf{0.1797 \pm 0.0014}$ & $\mathbf{0.631 \pm 0.004}$ & $\mathbf{0.142 \pm 0.011}$ & $\mathbf{0.697 \pm 0.004}$ \\ \hline
\multirow{3}{*}{$\mathsf{yeast}$} & CS-DPP-None & ${0.4992 \pm 0.0014}$ & ${0.622 \pm 0.001}$ & ${0.424 \pm 0.002}$ & ${0.737 \pm 0.001}$ \\ 
 & CS-DPP-PBT & $\mathbf{0.2139 \pm 0.0006}$ & ${0.389 \pm 0.001}$ & $\mathbf{0.017 \pm 0.001}$ & $\mathbf{0.495 \pm 0.001}$ \\ 
 & CS-DPP-PBC & $\mathbf{0.2144 \pm 0.0005}$ & $\mathbf{0.385 \pm 0.001}$ & $\mathbf{0.016 \pm 0.001}$ & $\mathbf{0.497 \pm 0.001}$ \\ \hline
\end{tabular}}

\caption{CS-DPP with PBC vs. PBT vs. None, $M$ = 25\% of $K$}
\label{tbl:pb_comp_m25}
\end{table*}

\begin{table*}[h]
\centering
    \resizebox{0.75\textwidth}{!}{
\begin{tabular}{cccccc }\hline
$\mathsf{Dataset}$ & $\mathsf{Alg.}$ & $\mathsf{Hamm. \, loss}$ & $\mathsf{F1 \, loss}$ & $\mathsf{Acc. \, loss}$ & $\mathsf{Norm. \, rank \, loss}$ \\ \hline
\multirow{3}{*}{$\mathsf{CAL500}$} & CS-DPP-None & ${0.4141 \pm 0.0176}$ & ${0.735 \pm 0.002}$ & ${0.398 \pm 0.002}$ & ${0.844 \pm 0.002}$ \\ 
 & CS-DPP-PBT & $\mathbf{0.1487 \pm 0.0002}$ & $\mathbf{0.602 \pm 0.001}$ & $\mathbf{0.154 \pm 0.001}$ & $\mathbf{0.752 \pm 0.001}$ \\ 
 & CS-DPP-PBC & $\mathbf{0.1490 \pm 0.0002}$ & $\mathbf{0.602 \pm 0.001}$ & $\mathbf{0.151 \pm 0.001}$ & $\mathbf{0.751 \pm 0.001}$ \\ \hline
\multirow{3}{*}{$\mathsf{Corel5k}$} & CS-DPP-None & ${0.5014 \pm 0.0017}$ & ${0.969 \pm 0.000}$ & ${0.369 \pm 0.001}$ & ${0.986 \pm 0.000}$ \\ 
 & CS-DPP-PBT & $\mathbf{0.0101 \pm 0.0000}$ & $\mathbf{0.843 \pm 0.000}$ & $\mathbf{0.214 \pm 0.001}$ & $\mathbf{0.901 \pm 0.000}$ \\ 
 & CS-DPP-PBC & $\mathbf{0.0101 \pm 0.0000}$ & $\mathbf{0.842 \pm 0.001}$ & $\mathbf{0.213 \pm 0.000}$ & ${0.903 \pm 0.001}$ \\ \hline
\multirow{3}{*}{$\mathsf{emotions}$} & CS-DPP-None & ${0.4941 \pm 0.0029}$ & ${0.631 \pm 0.003}$ & ${0.386 \pm 0.004}$ & ${0.729 \pm 0.002}$ \\ 
 & CS-DPP-PBT & $\mathbf{0.2308 \pm 0.0014}$ & $\mathbf{0.381 \pm 0.002}$ & $\mathbf{0.034 \pm 0.003}$ & $\mathbf{0.481 \pm 0.002}$ \\ 
 & CS-DPP-PBC & $\mathbf{0.2306 \pm 0.0012}$ & $\mathbf{0.377 \pm 0.002}$ & $\mathbf{0.033 \pm 0.003}$ & $\mathbf{0.481 \pm 0.002}$ \\ \hline
\multirow{3}{*}{$\mathsf{enron}$} & CS-DPP-None & ${0.4953 \pm 0.0016}$ & ${0.821 \pm 0.003}$ & ${0.385 \pm 0.002}$ & ${0.889 \pm 0.002}$ \\ 
 & CS-DPP-PBT & $\mathbf{0.0626 \pm 0.0002}$ & $\mathbf{0.523 \pm 0.001}$ & $\mathbf{0.130 \pm 0.001}$ & $\mathbf{0.636 \pm 0.001}$ \\ 
 & CS-DPP-PBC & ${0.0643 \pm 0.0001}$ & $\mathbf{0.522 \pm 0.001}$ & $\mathbf{0.129 \pm 0.001}$ & $\mathbf{0.636 \pm 0.001}$ \\ \hline
\multirow{3}{*}{$\mathsf{mediamill}$} & CS-DPP-None & ${0.4907 \pm 0.0018}$ & ${0.895 \pm 0.008}$ & ${0.426 \pm 0.001}$ & ${0.838 \pm 0.019}$ \\ 
 & CS-DPP-PBT & $\mathbf{0.0308 \pm 0.0000}$ & $\mathbf{0.457 \pm 0.000}$ & ${0.062 \pm 0.000}$ & ${0.581 \pm 0.000}$ \\ 
 & CS-DPP-PBC & $\mathbf{0.0307 \pm 0.0000}$ & $\mathbf{0.457 \pm 0.000}$ & $\mathbf{0.059 \pm 0.000}$ & $\mathbf{0.581 \pm 0.000}$ \\ \hline
\multirow{3}{*}{$\mathsf{medical}$} & CS-DPP-None & ${0.4177 \pm 0.0370}$ & ${0.907 \pm 0.002}$ & ${0.368 \pm 0.002}$ & ${0.944 \pm 0.002}$ \\ 
 & CS-DPP-PBT & ${0.0136 \pm 0.0001}$ & $\mathbf{0.252 \pm 0.002}$ & $\mathbf{0.021 \pm 0.001}$ & $\mathbf{0.303 \pm 0.002}$ \\ 
 & CS-DPP-PBC & $\mathbf{0.0130 \pm 0.0001}$ & $\mathbf{0.250 \pm 0.002}$ & $\mathbf{0.019 \pm 0.001}$ & $\mathbf{0.299 \pm 0.002}$ \\ \hline
\multirow{3}{*}{$\mathsf{nuswide}$} & CS-DPP-None & ${0.4972 \pm 0.0007}$ & ${0.940 \pm 0.004}$ & ${0.528 \pm 0.001}$ & ${0.964 \pm 0.002}$ \\ 
 & CS-DPP-PBT & $\mathbf{0.0201 \pm 0.0000}$ & $\mathbf{0.648 \pm 0.000}$ & ${0.307 \pm 0.000}$ & $\mathbf{0.674 \pm 0.000}$ \\ 
 & CS-DPP-PBC & $\mathbf{0.0201 \pm 0.0000}$ & $\mathbf{0.648 \pm 0.000}$ & $\mathbf{0.304 \pm 0.000}$ & $\mathbf{0.675 \pm 0.000}$ \\ \hline
\multirow{3}{*}{$\mathsf{scene}$} & CS-DPP-None & ${0.5015 \pm 0.0012}$ & ${0.745 \pm 0.001}$ & ${0.385 \pm 0.002}$ & ${0.832 \pm 0.001}$ \\ 
 & CS-DPP-PBT & $\mathbf{0.1731 \pm 0.0010}$ & $\mathbf{0.554 \pm 0.003}$ & ${0.125 \pm 0.005}$ & $\mathbf{0.626 \pm 0.004}$ \\ 
 & CS-DPP-PBC & $\mathbf{0.1720 \pm 0.0015}$ & $\mathbf{0.558 \pm 0.003}$ & $\mathbf{0.104 \pm 0.009}$ & $\mathbf{0.623 \pm 0.004}$ \\ \hline
\multirow{3}{*}{$\mathsf{yeast}$} & CS-DPP-None & ${0.4982 \pm 0.0011}$ & ${0.630 \pm 0.001}$ & ${0.413 \pm 0.001}$ & ${0.745 \pm 0.001}$ \\ 
 & CS-DPP-PBT & $\mathbf{0.2077 \pm 0.0003}$ & $\mathbf{0.382 \pm 0.001}$ & $\mathbf{0.024 \pm 0.001}$ & $\mathbf{0.493 \pm 0.001}$ \\ 
 & CS-DPP-PBC & $\mathbf{0.2079 \pm 0.0003}$ & $\mathbf{0.382 \pm 0.001}$ & $\mathbf{0.026 \pm 0.001}$ & $\mathbf{0.492 \pm 0.001}$ \\ \hline
\end{tabular}}

\caption{CS-DPP with PBC vs. PBT vs. None, $M$ = 50\% of $K$}
\label{tbl:pb_comp_m50}
\end{table*}

\subsection{Experiments on Cost-sensitivity}
\label{appendix:exp:cost}
We report the complete results of on all datasets with respect to all four cost functions in Table~\ref{tbl:cs_comp_m10} to Table~\ref{tbl:cs_comp_m50},
where the best values (the lowest) are marked in bold.
These complete results validate the conclusion in Section~\ref{sec:exp:cost}. 
\begin{table*}[h]
\centering
    \resizebox{0.75\textwidth}{!}{
\begin{tabular}{cccccc }\hline
$\mathsf{Dataset}$ & $\mathsf{Alg.}$ & $\mathsf{Hamm. \, loss}$ & $\mathsf{F1 \, loss}$ & $\mathsf{Acc. \, loss}$ & $\mathsf{Norm. \, rank \, loss}$ \\ \hline
\multirow{4}{*}{$\mathsf{CAL500}$} & O-CS & ${0.1610 \pm 0.0006}$ & ${0.953 \pm 0.003}$ & ${0.497 \pm 0.001}$ & ${0.971 \pm 0.002}$ \\ 
 & O-RAND & ${0.4042 \pm 0.0052}$ & ${0.750 \pm 0.004}$ & ${0.397 \pm 0.006}$ & ${0.858 \pm 0.004}$ \\ 
 & DPP & $\mathbf{0.1453 \pm 0.0001}$ & ${0.654 \pm 0.002}$ & ${0.399 \pm 0.001}$ & ${0.787 \pm 0.001}$ \\ 
 & CS-DPP & $\mathbf{0.1454 \pm 0.0002}$ & $\mathbf{0.603 \pm 0.001}$ & $\mathbf{0.144 \pm 0.002}$ & $\mathbf{0.748 \pm 0.001}$ \\ \hline
\multirow{4}{*}{$\mathsf{Corel5k}$} & O-CS & ${0.0117 \pm 0.0000}$ & ${0.926 \pm 0.002}$ & ${0.470 \pm 0.001}$ & ${0.949 \pm 0.001}$ \\ 
 & O-RAND & ${0.3734 \pm 0.0044}$ & ${0.980 \pm 0.001}$ & ${0.393 \pm 0.013}$ & ${0.990 \pm 0.000}$ \\ 
 & DPP & $\mathbf{0.0100 \pm 0.0000}$ & ${0.918 \pm 0.001}$ & ${0.470 \pm 0.000}$ & ${0.943 \pm 0.000}$ \\ 
 & CS-DPP & $\mathbf{0.0100 \pm 0.0000}$ & $\mathbf{0.850 \pm 0.001}$ & $\mathbf{0.237 \pm 0.001}$ & $\mathbf{0.910 \pm 0.000}$ \\ \hline
\multirow{4}{*}{$\mathsf{emotions}$} & O-CS & $\mathbf{0.3338 \pm 0.0073}$ & ${0.900 \pm 0.014}$ & ${0.508 \pm 0.006}$ & ${0.924 \pm 0.009}$ \\ 
 & O-RAND & ${0.3847 \pm 0.0099}$ & ${0.621 \pm 0.033}$ & ${0.363 \pm 0.020}$ & ${0.683 \pm 0.021}$ \\ 
 & DPP & $\mathbf{0.3335 \pm 0.0042}$ & $\mathbf{0.428 \pm 0.003}$ & ${0.223 \pm 0.009}$ & $\mathbf{0.558 \pm 0.004}$ \\ 
 & CS-DPP & $\mathbf{0.3301 \pm 0.0012}$ & ${0.450 \pm 0.007}$ & $\mathbf{0.133 \pm 0.023}$ & $\mathbf{0.560 \pm 0.009}$ \\ \hline
\multirow{4}{*}{$\mathsf{enron}$} & O-CS & ${0.0739 \pm 0.0006}$ & ${0.885 \pm 0.010}$ & ${0.463 \pm 0.004}$ & ${0.927 \pm 0.009}$ \\ 
 & O-RAND & ${0.3907 \pm 0.0090}$ & ${0.867 \pm 0.007}$ & ${0.320 \pm 0.015}$ & ${0.923 \pm 0.004}$ \\ 
 & DPP & $\mathbf{0.0563 \pm 0.0001}$ & ${0.552 \pm 0.003}$ & ${0.304 \pm 0.001}$ & ${0.646 \pm 0.002}$ \\ 
 & CS-DPP & $\mathbf{0.0565 \pm 0.0001}$ & $\mathbf{0.528 \pm 0.002}$ & $\mathbf{0.132 \pm 0.001}$ & $\mathbf{0.638 \pm 0.001}$ \\ \hline
\multirow{4}{*}{$\mathsf{mediamill}$} & O-CS & ${0.0485 \pm 0.0011}$ & ${0.821 \pm 0.025}$ & ${0.454 \pm 0.009}$ & ${0.868 \pm 0.014}$ \\ 
 & O-RAND & ${0.3737 \pm 0.0070}$ & ${0.899 \pm 0.007}$ & ${0.391 \pm 0.020}$ & ${0.950 \pm 0.003}$ \\ 
 & DPP & $\mathbf{0.0308 \pm 0.0000}$ & ${0.474 \pm 0.000}$ & ${0.307 \pm 0.000}$ & ${0.594 \pm 0.000}$ \\ 
 & CS-DPP & $\mathbf{0.0308 \pm 0.0000}$ & $\mathbf{0.460 \pm 0.000}$ & $\mathbf{0.072 \pm 0.002}$ & $\mathbf{0.582 \pm 0.000}$ \\ \hline
\multirow{4}{*}{$\mathsf{medical}$} & O-CS & ${0.0272 \pm 0.0006}$ & ${0.819 \pm 0.016}$ & ${0.397 \pm 0.004}$ & ${0.840 \pm 0.017}$ \\ 
 & O-RAND & ${0.3674 \pm 0.0093}$ & ${0.924 \pm 0.005}$ & ${0.301 \pm 0.019}$ & ${0.959 \pm 0.003}$ \\ 
 & DPP & $\mathbf{0.0204 \pm 0.0002}$ & ${0.602 \pm 0.008}$ & ${0.311 \pm 0.005}$ & ${0.628 \pm 0.008}$ \\ 
 & CS-DPP & $\mathbf{0.0204 \pm 0.0002}$ & $\mathbf{0.508 \pm 0.006}$ & $\mathbf{0.096 \pm 0.003}$ & $\mathbf{0.549 \pm 0.007}$ \\ \hline
\multirow{4}{*}{$\mathsf{nuswide}$} & O-CS & ${0.0239 \pm 0.0004}$ & ${0.746 \pm 0.005}$ & ${0.600 \pm 0.003}$ & ${0.741 \pm 0.003}$ \\ 
 & O-RAND & ${0.3707 \pm 0.0107}$ & ${0.956 \pm 0.002}$ & ${0.532 \pm 0.013}$ & ${0.973 \pm 0.002}$ \\ 
 & DPP & $\mathbf{0.0201 \pm 0.0000}$ & ${0.673 \pm 0.000}$ & ${0.580 \pm 0.000}$ & ${0.691 \pm 0.000}$ \\ 
 & CS-DPP & $\mathbf{0.0201 \pm 0.0000}$ & $\mathbf{0.648 \pm 0.000}$ & $\mathbf{0.358 \pm 0.001}$ & $\mathbf{0.675 \pm 0.000}$ \\ \hline
\multirow{4}{*}{$\mathsf{scene}$} & O-CS & ${0.2168 \pm 0.0047}$ & ${0.920 \pm 0.010}$ & ${0.491 \pm 0.003}$ & ${0.902 \pm 0.009}$ \\ 
 & O-RAND & ${0.3711 \pm 0.0172}$ & $\mathbf{0.743 \pm 0.030}$ & $\mathbf{0.295 \pm 0.029}$ & $\mathbf{0.782 \pm 0.009}$ \\ 
 & DPP & $\mathbf{0.1797 \pm 0.0001}$ & ${0.999 \pm 0.000}$ & ${0.500 \pm 0.000}$ & ${0.999 \pm 0.000}$ \\ 
 & CS-DPP & $\mathbf{0.1797 \pm 0.0001}$ & $\mathbf{0.724 \pm 0.002}$ & $\mathbf{0.231 \pm 0.016}$ & $\mathbf{0.796 \pm 0.002}$ \\ \hline
\multirow{4}{*}{$\mathsf{yeast}$} & O-CS & ${0.3077 \pm 0.0021}$ & ${0.885 \pm 0.018}$ & ${0.490 \pm 0.002}$ & ${0.926 \pm 0.014}$ \\ 
 & O-RAND & ${0.4162 \pm 0.0096}$ & ${0.596 \pm 0.008}$ & ${0.376 \pm 0.018}$ & ${0.702 \pm 0.014}$ \\ 
 & DPP & $\mathbf{0.2314 \pm 0.0014}$ & ${0.463 \pm 0.005}$ & ${0.340 \pm 0.003}$ & ${0.597 \pm 0.005}$ \\ 
 & CS-DPP & $\mathbf{0.2307 \pm 0.0011}$ & $\mathbf{0.433 \pm 0.003}$ & $\mathbf{0.003 \pm 0.000}$ & $\mathbf{0.541 \pm 0.003}$ \\ \hline
\end{tabular}}

\caption{CS-DPP vs Others, $M$ = 10\% of $K$}
\label{tbl:cs_comp_m10}
\end{table*}

\begin{table*}[h]
\centering
    
\caption{CS-DPP vs Others, $M$ = 25\% of $K$}
\label{tbl:cs_comp_m25}
\end{table*}

\begin{table*}[h]
\centering
    
\caption{CS-DPP vs Others, $M$ = 50\% of $K$}
\label{tbl:cs_comp_m50}
\end{table*}

\subsection{Experiments on Effect of Label Orders}
\label{appendix:exp:order}
The complete average results and the corresponding standard deviations of CS-DPP run on 50 random label orders are reported in Table \ref{tbl:cost_order_full}.
The results indicate that the standard deviation over the average results of 50 random orders are of $10^{-3}$ scale generally,
indicating that our CS-DPP is relatively not sensitive to the change of label order.
On the other hand, the results of CS-DPP have comparatively large deviation on several datasets for some cost functions,
such as the Normalized rank loss on dataset \textit{emotions} with $M = 10\%$ of $K$.
We attribute the reason to the instability of interaction between the randomness of $\Mv P_t$ and different label orders
based on the fact that larger deviations are observed only when $M = 10\%$ of $K$.
\begin{table*}[h]
\centering
\resizebox{0.9\textwidth}{!}{
\begin{tabular}{cccccc }\hline
$\mathsf{Dataset}$ & $\mathsf{Reduced \, Dim.}$ & $\mathsf{Hamm. \, loss}$ & $\mathsf{F1 \, loss}$ & $\mathsf{Acc. \, loss}$ & $\mathsf{Norm. \, rank \, loss}$ \\ \hline
\multirow{3}{*}{$\mathsf{CAL500}$} & $M = 10\%$ of $K$ & ${0.1458 \pm 0.00019}$ & ${0.5914 \pm 0.00108}$ & ${0.1247 \pm 0.00224}$ & ${0.7388 \pm 0.00105}$ \\ 
 & $M = 25\%$ of $K$ & ${0.1489 \pm 0.00012}$ & ${0.5956 \pm 0.00110}$ & ${0.1321 \pm 0.00210}$ & ${0.7428 \pm 0.00131}$ \\ 
 & $M = 50\%$ of $K$ & ${0.1503 \pm 0.00009}$ & ${0.5949 \pm 0.00101}$ & ${0.1371 \pm 0.00222}$ & ${0.7426 \pm 0.00127}$ \\ \hline
\multirow{3}{*}{$\mathsf{Corel5k}$} & $M = 10\%$ of $K$ & ${0.0102 \pm 0.00000}$ & ${0.8379 \pm 0.00175}$ & ${0.2382 \pm 0.00193}$ & ${0.9026 \pm 0.00138}$ \\ 
 & $M = 25\%$ of $K$ & ${0.0103 \pm 0.00000}$ & ${0.8248 \pm 0.00174}$ & ${0.2102 \pm 0.00102}$ & ${0.8936 \pm 0.00161}$ \\ 
 & $M = 50\%$ of $K$ & ${0.0102 \pm 0.00000}$ & ${0.8186 \pm 0.00138}$ & ${0.1991 \pm 0.00123}$ & ${0.8914 \pm 0.00152}$ \\ \hline
\multirow{3}{*}{$\mathsf{emotions}$} & $M = 10\%$ of $K$ & ${0.3421 \pm 0.00167}$ & ${0.4511 \pm 0.00525}$ & ${0.0745 \pm 0.08548}$ & ${0.5881 \pm 0.02669}$ \\ 
 & $M = 25\%$ of $K$ & ${0.2743 \pm 0.00000}$ & ${0.3964 \pm 0.00476}$ & ${0.0235 \pm 0.00597}$ & ${0.5068 \pm 0.00653}$ \\ 
 & $M = 50\%$ of $K$ & ${0.2324 \pm 0.00000}$ & ${0.3809 \pm 0.00450}$ & ${0.0237 \pm 0.00244}$ & ${0.4858 \pm 0.00463}$ \\ \hline
\multirow{3}{*}{$\mathsf{enron}$} & $M = 10\%$ of $K$ & ${0.0562 \pm 0.00020}$ & ${0.5421 \pm 0.00335}$ & ${0.1432 \pm 0.00333}$ & ${0.6573 \pm 0.00360}$ \\ 
 & $M = 25\%$ of $K$ & ${0.0600 \pm 0.00011}$ & ${0.5392 \pm 0.00291}$ & ${0.1364 \pm 0.00244}$ & ${0.6561 \pm 0.00332}$ \\ 
 & $M = 50\%$ of $K$ & ${0.0632 \pm 0.00009}$ & ${0.5428 \pm 0.00293}$ & ${0.1305 \pm 0.00216}$ & ${0.6627 \pm 0.00316}$ \\ \hline
\multirow{3}{*}{$\mathsf{mediamill}$} & $M = 10\%$ of $K$ & ${0.0309 \pm 0.00001}$ & ${0.4564 \pm 0.00037}$ & ${0.0617 \pm 0.00108}$ & ${0.5790 \pm 0.00049}$ \\ 
 & $M = 25\%$ of $K$ & ${0.0308 \pm 0.00000}$ & ${0.4535 \pm 0.00030}$ & ${0.0597 \pm 0.00062}$ & ${0.5756 \pm 0.00022}$ \\ 
 & $M = 50\%$ of $K$ & ${0.0308 \pm 0.00000}$ & ${0.4534 \pm 0.00027}$ & ${0.0565 \pm 0.00026}$ & ${0.5755 \pm 0.00027}$ \\ \hline
\multirow{3}{*}{$\mathsf{medical}$} & $M = 10\%$ of $K$ & ${0.0202 \pm 0.00014}$ & ${0.5246 \pm 0.01649}$ & ${0.0949 \pm 0.00836}$ & ${0.5764 \pm 0.02145}$ \\ 
 & $M = 25\%$ of $K$ & ${0.0150 \pm 0.00010}$ & ${0.3416 \pm 0.00815}$ & ${0.0337 \pm 0.00431}$ & ${0.4026 \pm 0.00942}$ \\ 
 & $M = 50\%$ of $K$ & ${0.0130 \pm 0.00003}$ & ${0.2783 \pm 0.00618}$ & ${0.0201 \pm 0.00276}$ & ${0.3361 \pm 0.00979}$ \\ \hline
\multirow{3}{*}{$\mathsf{nuswide}$} & $M = 10\%$ of $K$ & ${0.0201 \pm 0.00000}$ & ${0.6338 \pm 0.00064}$ & ${0.3394 \pm 0.00294}$ & ${0.6627 \pm 0.00063}$ \\ 
 & $M = 25\%$ of $K$ & ${0.0201 \pm 0.00000}$ & ${0.6305 \pm 0.00057}$ & ${0.3124 \pm 0.00189}$ & ${0.6600 \pm 0.00045}$ \\ 
 & $M = 50\%$ of $K$ & ${0.0201 \pm 0.00000}$ & ${0.6290 \pm 0.00035}$ & ${0.2945 \pm 0.00076}$ & ${0.6588 \pm 0.00042}$ \\ \hline
\multirow{3}{*}{$\mathsf{scene}$} & $M = 10\%$ of $K$ & ${0.2837 \pm 0.00000}$ & ${0.7433 \pm 0.00184}$ & ${0.1732 \pm 0.00593}$ & ${0.7917 \pm 0.00111}$ \\ 
 & $M = 25\%$ of $K$ & ${0.1873 \pm 0.00000}$ & ${0.6387 \pm 0.00199}$ & ${0.2265 \pm 0.02788}$ & ${0.6882 \pm 0.00196}$ \\ 
 & $M = 50\%$ of $K$ & ${0.1723 \pm 0.00005}$ & ${0.5571 \pm 0.00206}$ & ${0.1708 \pm 0.02179}$ & ${0.6138 \pm 0.00221}$ \\ \hline
\multirow{3}{*}{$\mathsf{yeast}$} & $M = 10\%$ of $K$ & ${0.2296 \pm 0.00010}$ & ${0.4518 \pm 0.00920}$ & ${0.0064 \pm 0.00081}$ & ${0.5448 \pm 0.02253}$ \\ 
 & $M = 25\%$ of $K$ & ${0.2162 \pm 0.00009}$ & ${0.3841 \pm 0.00200}$ & ${0.0170 \pm 0.00242}$ & ${0.4971 \pm 0.00379}$ \\ 
 & $M = 50\%$ of $K$ & ${0.2092 \pm 0.00001}$ & ${0.3784 \pm 0.00107}$ & ${0.0232 \pm 0.00158}$ & ${0.4901 \pm 0.00124}$ \\ \hline
\end{tabular}}
\caption{Results of CS-DPP on 50 random label orders}
\label{tbl:cost_order_full}
\end{table*}

%
%

\end{document}